\documentclass[manuscript,screen,review=false]{acmart}
\usepackage{amsmath,amsfonts}

\usepackage{amssymb}
\usepackage{amsthm}
\usepackage[linesnumbered,ruled]{algorithm2e}
\usepackage{algpseudocode}
\usepackage{array}
\usepackage[caption=false,font=normalsize,labelfont=sf,textfont=sf]{subfig}
\usepackage{textcomp}
\usepackage{stfloats}
\usepackage{url}
\usepackage{color}
\usepackage{verbatim}
\usepackage{graphicx}
\usepackage{booktabs}
\usepackage{multirow}
\usepackage{multicol}
\usepackage{makecell}
\usepackage{float}

\usepackage{marvosym}

\AtBeginDocument{%
  }

\setcopyright{acmlicensed}
\copyrightyear{2025}
\acmYear{2025}
\acmDOI{XXXXXXX.XXXXXXX}

\acmConference[TOIS 2025]{ACM Transactions on Information Systems}{}{}
\acmISBN{978-1-4503-XXXX-X/18/06}

\begin{document}

\title{Deep Hashing with Semantic Hash Centers for Image Retrieval}

\author{Li Chen}
\email{cc752424640@buaa.edu.cn}
\affiliation{%
  \institution{Beihang University}
  \state{Beijing}
  \country{China}
}

\author{Rui Liu}
\affiliation{%
  \institution{Beihang University}
  \state{Beijing}
  \country{China}
}

\author{Yuxiang Zhou}
\affiliation{%
  \institution{ Beijing University of Posts and Telecommunications}
  \state{Beijing}
  \country{China}
}

\author{Xudong Ma}
\email{macaronlin@buaa.edu.cn}
\affiliation{%
  \institution{Beihang University}
  \state{Beijing}
  \country{China}
}

\author{Yong Chen}
\email{alphawolf.chen@gmail.com}
\authornote{Yong Chen is the Corresponding author (yong.chen@bupt.edu.cn).}
\affiliation{%
  \institution{ Beijing University of Posts and Telecommunications}
  \state{Beijing}
  \country{China}
}

\author{Dell Zhang}
\email{dell.z@ieee.org}
\authornote{Dell Zhang is also the Corresponding author (dell.z@ieee.org).}
\affiliation{%
  \institution{TeleAI}
  \city{Shanghai}
  \country{China}}

\renewcommand{\shortauthors}{Li Chen, Rui Liu, Yuxiang Zhou, Xudong Ma, Yong Chen, Dell Zhang}

\begin{abstract}
Deep hashing presents an effective strategy for large-scale image retrieval. 
Current hashing methods are generally categorized by their supervision types: point-wise, pair-wise, and list-wise. 
Recent advancements in point-wise methods (e.g., CSQ, MDS) have significantly enhanced retrieval performance across diverse datasets by pre-assigning a hash center to each class, thereby improving the discriminability of the resultant hash codes. 
However, these methods employ purely data-independent algorithms for generating hash centers, overlooking the semantic connections between different classes, which, we argue, could degrade retrieval performance. 
To tackle this problem, this paper expands on the newly emerged concept of ``hash centers'' to introduce ``\emph{semantic} hash centers'', which posits that hash centers of semantically related classes should exhibit closer Hamming distances, while those of unrelated classes should be more distant. 
Based on this hypothesis, we propose a three-stage framework, termed SHC, to produce hash codes that preserve semantics. 
First, we build a classification network to detect semantic similarities between classes, and utilize a data-dependent approach to similarity calculation that can adapt to varied data distributions. 
Next, we develop a new optimization algorithm to generate semantic hash centers. This algorithm not only maintains semantic relatedness among hash centers but also integrates a constraint to ensure a minimum distance between them, addressing the issue of excessively proximate hash centers potentially impairing retrieval performance. 
Finally, we train a deep hashing network with the above generated semantic hash centers to convert each image into a binary hash code. 
Experiments on large-scale image retrieval across several public datasets demonstrate that SHC generates more discriminative hash codes, markedly enhancing retrieval performance. 
Specifically, in terms of the MAP@100, MAP@1000, and MAP@ALL metrics, SHC records average improvements of +7.26\%, +7.62\%, and +11.71\%, respectively, over the most competitive existing methods.
\end{abstract}

\ccsdesc[500]{Information systems~Top-k retrieval in databases}

\keywords{Learning to Hash, Hash Center, Quantization, Representation Learning, Image Retrieval.}


\maketitle

\section{Introduction}
With the onset of the big data era, images have become increasingly vital in social media, e-commerce, surveillance, and various other domains~\cite{LargeScaleImageRetrieval1, LargeScaleImageRetrieval2}.
This shift poses substantial challenges in managing the growing volume of images, necessitating the development of efficient and effective methods for indexing and searching large-scale image data repositories.
Image hashing, as a quantization technique~\cite{LearningToHash,PQ-Hervegou-TPAMI-2011,ITQ-Yunchao-TPAMI-2013}, has emerged as a practical solution to these challenges. 
By converting images into fixed-length binary codes, it not only cuts down on storage needs significantly but also enables rapid comparisons and retrievals via hardware-level XOR operations, and thus benefits a wide array of applications.

\begin{figure*}[ht]
	\centering
	\includegraphics[width=0.99\textwidth]{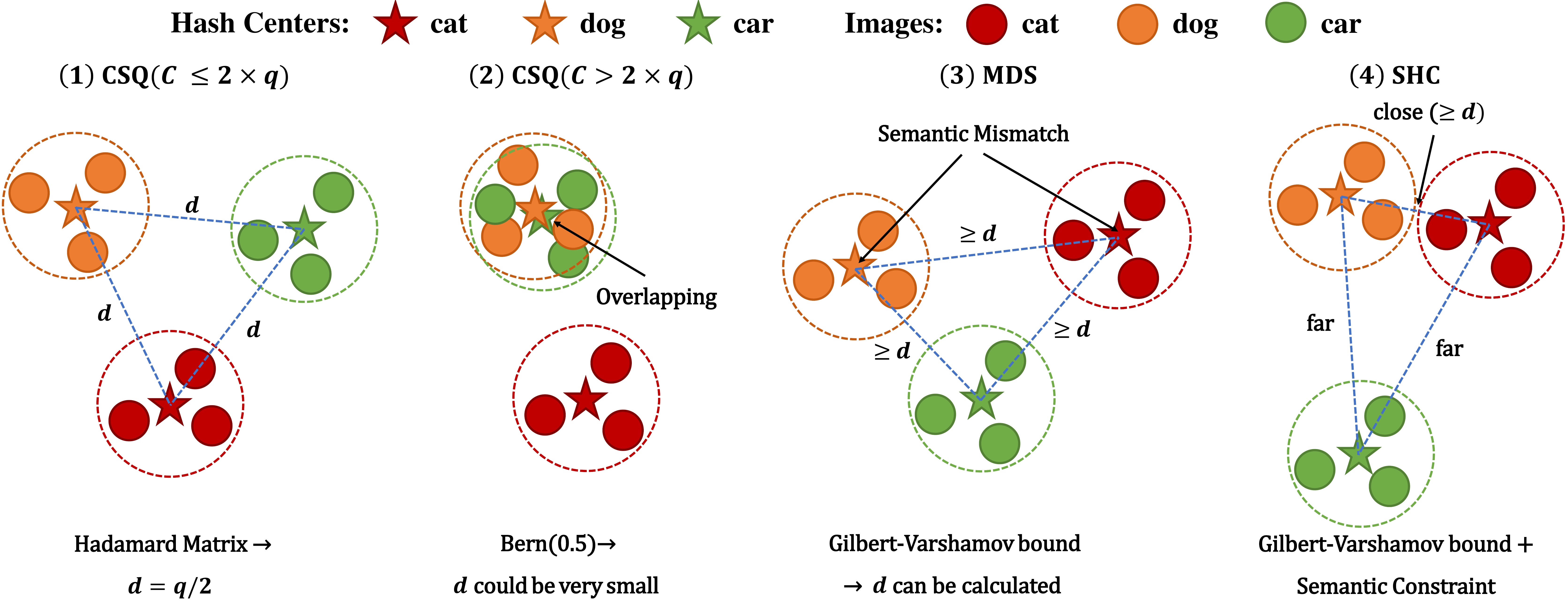}
	\caption{Illustration of hash center based image hashing approaches: CSQ (1\&2), MDS (3), and our SHC (4).}
	\label{fig:HashCenters}
\end{figure*}

Image hashing has undergone a series of advancements along with the evolution of technology~\cite{Survey-Jingdong-TPAMI-2018,Survey-Xiao-TKDD-2023}. 
Initially, this research field focused on data-independent, shallow feature-based hashing methods, such as Locality Sensitive Hashing (LSH)~\cite{LSH1, LSH2, LSH3}; but recently, it has transitioned to data-dependent, deep semantic feature-based deep hashing approaches, such as DPSH~\cite{DPSH,CNNs}, HashNet~\cite{HashNet}, ADSH~\cite{ADSH-QingYuan-AAAI-2018}, and CSQ~\cite{CSQ}. 
This transformation has vastly improved the capability of image hashing, from processing small datasets in its early stages to efficiently managing large-scale data sets, underscoring its practical utility.

Existing deep hashing methods can be broadly categorized into three types: pair-wise similarity-based methods~\cite{DPSH, HashNet, DCH, BNNH, IDHN, DHN}, list-wise similarity-based methods~\cite{DNNH, norouzi2012hamming, DTSH}, and point-wise methods~\cite{GreedyHash, LTH, CSQ, MDS, OrthoHash}. 
Both pair-wise and list-wise approaches primarily focus on the relationships among local images, which limits their effectiveness. 
To address this problem, CSQ~\cite{CSQ} introduces the novel concept of ``hash centers'', in the category of point-wise methods. 
This approach utilizes the Hadamard matrix to predefine hash centers in the Hamming space for each image class, aiming to align the hash codes of similar images with their respective hash centers. 
CSQ shows impressive performance when the number of categories (i.e., $C$) is less than twice the length of hash codes (i.e., $q$), as illustrated in Fig.~\ref{fig:HashCenters}(1).
However, challenges arise when the number of categories exceeds twice the length of hash codes. 
In these cases, the Hadamard matrix cannot produce a hash center for each individual category. 
Consequently, some hash centers are randomly generated according to the Bernoulli distribution. 
This can lead to hash centers being too close or even overlapping, which would result in complete misclassification of certain categories, as depicted in Fig.~\ref{fig:HashCenters}(2).

MDS~\cite{MDS} proposes an novel solution to the above mentioned issue. 
Unlike CSQ, MDS makes use of the Gilbert-Varshamov bound~\cite{GVbound} to determine the feasible minimum distance between hash centers, taking both the number of image categories and the length of hash codes into account to ensure optimal spacing between hash centers. 
MDS strives to eliminate the problem of overly close hash centers, particularly when dealing with a vast number of categories and a limited hash code length. 
As a result, MDS achieves superior performance compared to CSQ, as shown clearly in Fig.~\ref{fig:HashCenters}(3).

Nevertheless, both CSQ and MDS rely solely on mathematical optimality to generate hash centers, which results in a lack of semantic interpretability (see Fig.~\ref{fig:HashCenters}). 
These methods are agnostic to the actual data, and assume that a larger Hamming distance between hash centers translates straightforwardly into better retrieval accuracy. 
This presumption overlooks the crucial similarity relationships between different categories. 
Therefore, the hash centers generated by these methods are short of discriminative power and are arbitrarily assigned to categories, which undermines the effectiveness of image retrieval. 
We argue that this approach is flawed and substantially impairs the potential for optimal retrieval performance.

Therefore, in this paper, we introduce SHC, which posits that in a confined Hamming space, the quality of hash centers should not merely depend on maximizing the distances between them; 
it is imperative to consider the semantic relationships among categories as well. 
This new approach ensures that while maintaining proper distance between hash centers, they are also endowed with semantic discriminatory capabilities. 
Specifically, semantically similar categories ought to possess hash centers with shorter Hamming distances, whereas semantically dissimilar categories should exhibit longer distances. 
We designate these as ``\emph{semantic} hash centers'', as illustrated in Fig.~\ref{fig:HashCenters}(4).

Our main contributions are threefold.
\begin{itemize}
	\item We advance the novel concept of ``hash centers'' to the more explainable ``\emph{semantic} hash centers'', and introduce a data-dependent similarity calculation method that utilizes the inherent semantic content of images, moving beyond the conventional use of categorical labels. This approach enhances adaptability across varied data distributions.
	\item We develop a comprehensive optimization framework tailored for generating semantic hash centers. This framework aims to preserve semantic relationships while enforcing minimum distance constraints. By incorporating proxy variables and harnessing the ALM optimization scheme, we have successfully broken down the NP-hard problem into smaller, more manageable sub-optimization tasks, yielding impressive optimization results.
	\item Extensive experiments on multiple datasets underscore the superiority of our SHC approach over other cutting-edge deep hashing methods, defining the new state of the art. Specifically, in terms of the evaluation metrics MAP@100, MAP@1000, and MAP@ALL, SHC demonstrates significant performance improvements, with average increases of +7.26\%, +7.62\%, and +11.71\%, respectively. Ablation studies further validate the efficacy of our semantic and minimum distance constraints, as well as the benefits of utilizing image semantics over categorical labels.
\end{itemize}

\section{Related Work}

Supervised deep hashing approaches use deep learning techniques to learn, from labeled data, hashing functions that map original high-dimensional data points to a low-dimensional hash code space, enabling efficient and accurate data retrieval. 
They can be categorized into pair-wise, list-wise, and point-wise methods, according to the way of semantic measurement.

\textbf{Pair-wise methods} establish the similarities among data points by using pairs of images. 
If two images share at least one common label, they are deemed similar and assigned a value of 1 in the pairwise similarity matrix; otherwise, they are considered dissimilar and assigned a value of 0. 
Some prominent pairwise approaches include DPSH~\cite{DPSH}, HashNet~\cite{HashNet}, DCH~\cite{DCH}, BNNH~\cite{BNNH}, and IDHN~\cite{IDHN}. DPSH marks the early efforts in this field, while HashNet adjusts the loss function to deal with the challenge of semantic imbalance. 
Besides, DCH introduces the concept of ``Hamming radius'' to transform the linear retrieval problem into an index lookup problem, thereby enhancing efficiency. In addition, BNNH replaces CNN with BNN to minimize the parameter overhead, making it much more computationally efficient. 
Lastly, IDHN quantifies the similarity of data pairs based on normalized semantic labels, significantly improving image retrieval performance on multi-label datasets.

\textbf{List-wise methods} rank or order data pairs based on their labels, where a set of samples (or a list) is considered simultaneously during training, rather than just pairs or individual points. 
DTSH~\cite{DTSH} is a typical list-wise method, which, unlike DPSH, uses triplet similarity to process images. In a triplet, the first image is similar to the second, but differs from the third. 
The training objective not only constrains the Hamming distances of similar image pairs to be close but also enforces a large distance between dissimilar image pairs, resulting in a superior performance compared to DPSH. 
DSRH~\cite{DSRH-Fang-CVPR-2015,DSRH-Ting-IJCAI-2016,DcmSRH-Xiaoqing-TMM-2023} is another representative ranking-based method that directly learns the ranking of multi-label similarities and proposes a novel loss function to optimize the multi-label ranking metric.

\textbf{Point-wise methods} rely on image categories as supervised signals. 
Prominent examples of such methods include GreedyHash~\cite{GreedyHash}, LTH~\cite{LTH}, CSQ~\cite{CSQ}, MDS~\cite{MDS}, and OrthoHash~\cite{OrthoHash}. 
GreedyHash uses a point-wise loss function that combines classification loss with penalty terms. During training, the sign function is seamlessly integrated, and a greedy backpropagation algorithm is employed to fine-tune network parameters. 
LTH adopts the familiar cross-entropy loss function, used in standard classification tasks. It reconnects image hash codes to labels and computes cross-entropy loss with the original labels, thus preserving semantic congruity between hash codes and image labels. 
CSQ introduces a global similarity metric and the concept of ``hash centers''. Each image category is assigned a unique hash center, maintaining a constant Hamming distance between them. CSQ encourages images of the same category to cluster around their respective centers, resulting in compact hash codes. 
MDS builds on CSQ, addressing the challenge of minimizing the distance between hash centers when the number of categories exceeds the capacity of Hadamard matrices. 
Orthohash introduces a deep hashing model with a unified learning objective, demonstrating that maximizing cosine similarity between continuous codes and their corresponding binary orthogonal codes ensures both discriminative hash codes and minimal quantization errors.

Our SHC method draws inspiration from the point-wise CSQ~\cite{CSQ} and MDS~\cite{MDS} approaches, particularly their innovative concept of ``Hash Center''. 
SHC expands upon this idea by introducing a novel concept called ``Semantic Hash Center'', which not only takes into account the distance constraints imposed on hash centers but also endows them with the ability to maintain semantic consistency.

\section{Method}

\begin{figure*}[ht]
    \centering
    \includegraphics[width=0.99\textwidth]{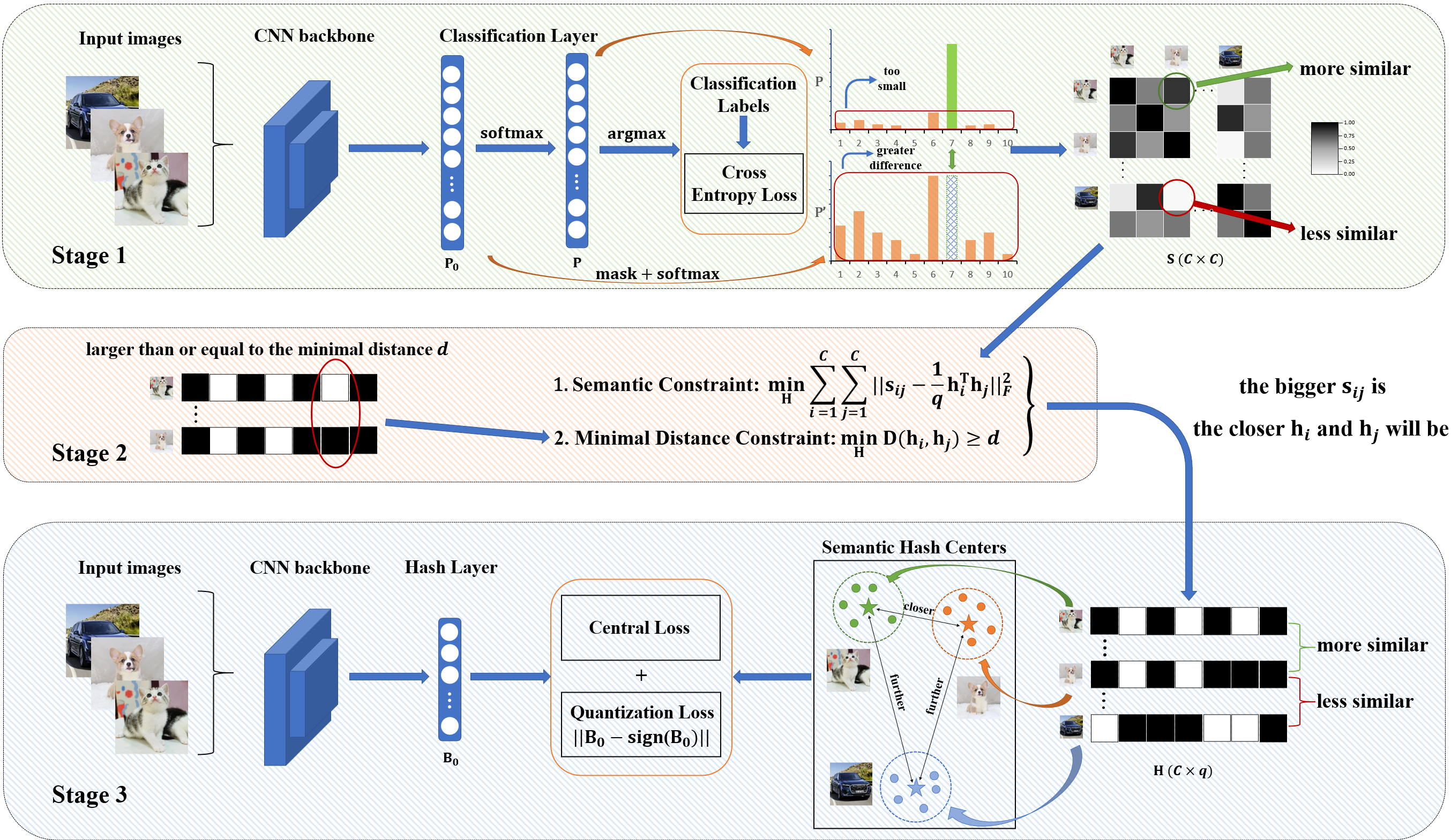}
    \caption{The three-stage architecture of our SHC method: similarity matrix construction, hash centers generation, and hashing network training.}
    \label{fig:architecture}
\end{figure*}

Image hashing aims to find a hash function $\text{F}: \mathbf{x} \to \{-1, +1\}^q$ that converts an image $\mathbf{x}$ into a $q$-bit vector, so that the more similar two images are, the smaller their Hamming distance will be. 

The recently appeared SOTA methods, e.g., CSQ~\cite{CSQ} and MDS~\cite{MDS}, are hash centers based optimization algorithms, which involve two stages, i.e., hash centers generation and hashing network training. 
They assume that the further their generated hash centers are, the better their performances will be; however, they neglect the semantic correlations between their generated hash centers.

Hence, we propose an enhanced deep hashing approach that employs a 3-stage pipeline, i.e., similarity matrix generation, hash centers generation, and hashing network training, as shown in Fig.~\ref{fig:architecture}, which not only tries to keep hash centers as far away from each other as possible, but also maintains semantic consistencies between them.

\subsection{Stage 1: Construct the Data-dependent Pairwise Similarity Matrix}

Commonly, there are 2 methods for constructing similarity matrix $\mathbf{S}=\{\mathbf{s}_{ij}\}^{N\times N}$, of which $\mathbf{s}_{ij}$ represents the similarity score between the $i$-th and $j$-th image, and $N$ denotes the total number of images. 
For the first method, if image $\mathbf{x}_i$ and image $\mathbf{x}_j$ share the same class label, then $\mathbf{s}_{ij}$ equals 1; or else $\mathbf{s}_{ij}$ equals 0.
For the second method, $\mathbf{s}_{ij}$ is usually defined as the cosine value of image $\mathbf{x}_i$ and image $\mathbf{x}_j$, i.e.,
\begin{equation}
\mathbf{s}_{ij}=\text{cosine}<\mathbf{x}_i,\mathbf{x}_j>=\frac{\mathbf{x}_i^T \mathbf{x}_j}{||\mathbf{x}_i||\cdot||\mathbf{x}_j||},
\end{equation}
where $\mathbf{x}_i$ ($i=1,2,\cdots,N$) can be computed as the semantic vector of its textual labels from pre-trained large language models such as BERT~\cite{BERT-Jacob-NAACL-2019} or GPT~\cite{GPT-Tom-NeurIPS-2020}.
These two methods have been testified to be useful in semantic preservations~\cite{HashNet,DPSH}; however, the first approach is coarse-grained with only $\{0,1\}$ values while the second one is not adaptive to various data distributions w.r.t. the same classes. 

Here, we devise a data-dependent similarity matrix construction method via conducting a pre-classification task on concrete image datasets.
More specifically, we utilize a pre-trained ResNet34~\cite{ResNet-Kaiming-CVPR-2016} network followed by a softmax layer for image classifications, i.e., 
\begin{equation}
\mathbf{p}_{0i} = \text{ResNet34}(\mathbf{x}_i),
\end{equation}
\begin{equation}
\mathbf{p}_i = \text{softmax}(\mathbf{p}_{0i}),
\end{equation}
where $\mathbf{p}_i$ is a predicted probability vector over the $C$ classes, and traditionally the image $\mathbf{x}_i$ is categorized to the $j$-th class with
\begin{equation}
j=\mathop{\arg\max}\limits_{k}\{\mathbf{p}_{ik}\}, k=1,2,\cdots,C.
\end{equation} 

However, other values of $\mathbf{p}_i$, i.e., $\mathbf{p}_{ik}$ (k$\ne$j), also covering rich semantics, are not used to the fullest. 
In fact, the larger $\mathbf{p}_{ik}$ is, the higher the similarity between the $i$-th and the $k$-th class can be treated. 
Usually, $\mathbf{p}_{ij}$ is the largest, and much larger than the other values $\mathbf{p}_{ik}$, as shown in Fig.~\ref{fig:architecture}; thus, 
we first mask the maximum value $\mathbf{p}_{0ij}$ of $\mathbf{p}_{0i}$, and then softmax it to $\mathbf{p}_i^{\prime}$, formulated as:
\begin{equation}
\hat{\mathbf{p}}_i \triangleq \text{mask}_j(\mathbf{p}_{0i}) \in \mathbb{R}^C,
\end{equation}
and
\begin{equation}
\mathbf{p}^{\prime}_i = \text{softmax}(\hat{\mathbf{p}}_i),
\end{equation}
where $\hat{\mathbf{p}}_{ik}=\mathbf{p}_{0ik}$ and $\hat{\mathbf{p}}_{ij}=-\text{INF}$.

For all the images in the $m$-th category, we average their similarity vectors, i.e.,
\begin{equation}
\mathbf{s}_m = \frac{1}{N_m}\sum_{l = 1}^{N_m} \mathbf{p}_l^{\prime},
\end{equation}
where $N_m$ represents the number of images in category $m$, and $\mathbf{s}_m$ represents the overall similarity between catetory $m$ and other categories.

Next, we normalize $\mathbf{s}_m$ to $(-1,+1)^C$-vector via:
\begin{equation}
\mathbf{s}_m = \frac{\mathbf{s}_m - \mathbf{s}_{m0}}{\max\{|\max(\mathbf{s}_m) - \mathbf{s}_{m0}|, |\min(\mathbf{s}_m)-\mathbf{s}_{m0}|\}},
\end{equation}
where
\begin{equation}
	\mathbf{s}_{m0}= \frac{1}{C}\sum_{l = 1}^{C} \mathbf{s}_{ml}.
\end{equation}

Since $\mathbf{s}_m$ denotes the overall similarity between the $m$-th class and the other classes, we can concatenate them as the whole pairwise similarity matrix $\mathbf{S}$:
\begin{equation}
\mathbf{S}=[\mathbf{s}_1,\mathbf{s}_2,\cdots,\mathbf{s}_C]^T \in \mathbb{R}^{C \times C},
\end{equation}
of which $\mathbf{s}_{ij}$ denotes the similarity between the $i$-th and the $j$-th class. 
Here, $\mathbf{S}$ is not symmetric and its diagonal elements does not equal 1 yet, then we further conduct the following operators for the final similarity matrix:
\begin{equation}
\mathbf{S} = \frac{\mathbf{S} + \mathbf{S^T}}{2},
\end{equation}
and
\begin{equation}
\mathbf{s}_{ii} = 1,
\end{equation}
which deserves most notice that it's data-dependent instead of labels-based, thus adaptive to various data distributions.

\subsection{Stage 2: Generate the Semantic Hash Centers}
\label{stage2}
In this part, we develop an optimization procedure to yield a set of hash centers, which are not only under the guidance of similarity matrix $\mathbf{S}$, but also mutually separated at least by a minimal distance $d$ that can be computed by the Gibert-Varshamov bound (see Theorem~\ref{theorem:Gilbert-Varshamov}). 

\subsubsection{The Minimal Distance Calculation}

The Gilbert-Varshamov bound lets us know that for $C$ hash centers, the Hamming distance between any two of them is larger than or at least equal to $d$.

\begin{theorem}[Gilbert-Varshamov
	Bound]\label{theorem:Gilbert-Varshamov}
	For $C$ $q$-bit binary vectors $\mathbf{h}_i \in \{-1,+1\}^q$ ($1 \le i \le C$), set the minimal Hamming distance of any two vectors as $d$, there is the Gilbert-Varshamov bound:
	\begin{equation}\label{GV-bound}
	\frac{2^q}{C} \le \sum_{i=0}^{d-1} \binom{q}{i}. 
	\end{equation}
\end{theorem}
\begin{proof}
	This theorem has been proven by Ref.~\cite{Rom-ErrorCorCodes-Docklady-1957}, which in fact provides a method to calculate the minimal distance $d$ given $C$ and $q$.
\end{proof}

To be specific, since $d$ is an integer in $\{1, 2, ..., q\}$, and according to Eq.~(\ref{GV-bound}), we can compute it easily by exhaustive search~\cite{MDS}.


\subsubsection{Optimization Objective}

For all images of $C$ classes, we consider to generate $C$ hash centers $\{\mathbf{h}_1, \mathbf{h}_2, \cdots , \mathbf{h}_C\}$ with $\mathbf{h}_i$ corresponding to the $i$-th class. 
Clearly, if the $i$-th and $j$-th class are more semantically similar, then their cosine score between $\mathbf{h}_i$ and $\mathbf{h}_j$ is much closer to $\mathbf{s}_{ij}$, which can be formulated as:
\begin{equation}\label{optimization:semantic}
\begin{split}
\mathop{\min}_{\mathbf{h}_i,i=1,\cdots,C} &\text{L}_{semantic}=\sum_{i=1}^{C}\sum_{j=1}^{C}||\mathbf{s}_{ij}-\frac{1}{q}\mathbf{h}_i^T \mathbf{h}_j||_2^2 \\
&\text{s.t.} \quad \mathbf{h}_i \in \{-1, +1\}^q;
\end{split}
\end{equation}
in addition, we want the distances between hash centers are as large as possible, and meanwhile the Hamming distance between any two hash centers is larger than or at least equal to the minimal distance $d$, i.e., 
\begin{equation}\label{optimization:distance}
\begin{split}
\mathop{\max}_{\mathbf{h}_i,i=1,\cdots,C}&\text{L}_{distance_o}=
\sum_{i=1}^{C} \sum_{j \neq i} ||\mathbf{h}_i-\mathbf{h}_j||_{Hamming} \\
&\text{s.t.} \quad
||\mathbf{h}_i-\mathbf{h}_j||_{Hamming} \ge d, (i \ne j).
\end{split}
\end{equation}

By the following Theorem~\ref{theorem:Hamming-Euclid}, the optimizaiton (\ref{optimization:distance}) is then converted to:
\begin{equation}\label{optimization:distance-simplified}
\begin{split}
\mathop{\min}_{\mathbf{h}_i,i=1,\cdots,C}&\text{L}_{distance}=
\sum_{i=1}^{C}\sum_{j \neq i} \mathbf{h}_i^T \mathbf{h}_j \\
&\text{s.t.} \quad
\mathbf{h}_i^T \mathbf{h}_j \le q-2d, (1\le i,j\le C, j\ne i).
\end{split}
\end{equation}

\begin{theorem}[Hamming-Euclid]\label{theorem:Hamming-Euclid}
	For any two binary vectors $\mathbf{h}_i,\mathbf{h}_j \in \{-1,+1\}^{q}$, there is:
	\begin{equation}\label{eq:Hamming_Euclid}
	\text{D}(\mathbf{h}_i, \mathbf{h}_j) \triangleq ||\mathbf{h}_i-\mathbf{h}_j||_{Hamming}=\frac{1}{2}(q-\mathbf{h}_i^T \mathbf{h}_j),
	\end{equation}
	where $||\cdot||_{Hamming}$ represents the number of non-zero elements for a given vector; for example, if $\mathbf{x}=[0,0,-2,-2,+2+2]^T$, then $||\mathbf{x}||_{Hamming}=4$.
\end{theorem}
\begin{proof}
	Let's first assume that:
	\begin{equation}
	\begin{split}
	\mathbf{x}=&\mathbf{h}_i - \mathbf{h}_j \\
	=&[x_1,\cdots,x_s,\cdots,x_q]^T,
	\end{split}
	\end{equation}
	where $x_s$ ($s=1,\cdots,q$) equals to 0 or -2 or +2. Denote the number of 0 and -2 and +2 in $\mathbf{x}$ as $n_0$, $n_{-2}$, $n_{+2}$, respectively. Then, there exists:
	\begin{equation}
	q=n_0 + n_{-2} + n_{+2}.
	\end{equation}
	
	Next, let's do some derivations:
	\begin{equation}
	||\mathbf{h}_i-\mathbf{h}_j||_{Hamming}=n_{-2} + n_{+2},
	\end{equation}
	and
	\begin{equation}
	\begin{split}
	\frac{1}{2}(q-\mathbf{h}_i^T \mathbf{h}_j)=&\frac{1}{2}(q-(n_0-(n_{-2} + n_{+2}))) \\
	=&\frac{1}{2}((q-n_0)+(n_{-2} + n_{+2})) \\
	=&\frac{1}{2}((n_{-2} + n_{+2})+(n_{-2} + n_{+2}))\\
	=&n_{-2} + n_{+2},
	\end{split}
	\end{equation}
	which finishes the proof of Eq.~(\ref{eq:Hamming_Euclid}), building the connection between the Hamming space and the Euclid space.
\end{proof}

Taking the optimization (\ref{optimization:semantic}) and (\ref{optimization:distance-simplified}) into considerations, we could obtain the overall optimization:
\begin{equation} \label{optimization:objective}
\begin{split}
\mathop{\min}_{\mathbf{H}} \text{L} & =\text{L}_{semantic} + \mu \times \text{L}_{distance}\\
& = ||\mathbf{S} - \frac{1}{q}\mathbf{H}^T\mathbf{H}||^{2}_{F} + \mu \sum_{i=1}^{C} \sum_{j \neq i} \mathbf{h}_{i}^{T}\mathbf{h}_j \\
&\text{s.t.}
\left \{
\begin{array}{ll}
\mathbf{h}_{i}^{T}\mathbf{h}_j \leq q - 2d, (1 \leq i,j \leq C, j \neq i), & \\
\mathbf{h}_i \in \{-1, +1\}^q, &
\end{array}
\right.
\end{split}
\end{equation}
where $\mathbf{S}$=$(\mathbf{s}_{ij})^{C \times C}$ represents the data-dependent similarity matrix, $\mathbf{H}=[\mathbf{h}_1, \mathbf{h}_2, \cdots, \mathbf{h}_C]$ denotes $C$ hash centers, $||\cdot||_F$ is the matrix's $\mathcal{F}$-norm, and $\mu$ is a hyper-parameter to balance the importance between the semantics and distances.

\subsubsection{Alternating Optimization Procedure}

The optimization (\ref{optimization:objective}) is a NP-hard problem because of the binary constraints $\mathbf{H} \in \{-1, +1\}^{q \times C}$. 
A feasible optimization method is to relax the discrete constraints; however, it will result in sub-optimal performance. 
Then, we devise a procedure that can optimize its variables in a mixed discrete-and-continuous manner.

First, we introduce an auxiliary variable $\mathbf{M}$ that satisfies:
\begin{equation}
\mathbf{M} = [\mathbf{m}_1, \mathbf{m}_2, \cdots, \mathbf{m}_C] = \mathbf{H},
\end{equation}
where $\mathbf{M} \in \mathbb{R}^{q \times C}$. 
Then let $\mathbf{k}_{ij}$ equals to $q - \mathbf{h}_{i}^{T}\mathbf{h}_j - 2d$, i.e.,
\begin{equation}
	\mathbf{k}_{ij}=q - \mathbf{h}_{i}^{T}\mathbf{h}_j - 2d.
\end{equation}
Next, we replace one $\mathbf{H}$ with $\mathbf{M}$ in Eq.~(\ref{optimization:objective}), and transform the optimization (\ref{optimization:objective}) into:
\begin{equation} \label{optimization:objective_2}
\begin{split}
\mathop{\min}_{\mathbf{H, M}} \text{L} & = ||\mathbf{S} - \frac{1}{q}\mathbf{H}^T\mathbf{M}||^{2}_{F} + \mu \sum_{i=1}^{C} \sum_{j \neq i} \mathbf{h}_{i}^{T}\mathbf{h}_j \\
&\text{s.t.}
\left \{
\begin{array}{lr}
    \mathbf{h}_i \in \{-1, +1\}^q, & \\
    \mathbf{m}_i \in \mathbb{R}^q, & \\
    \mathbf{h}_i = \mathbf{m}_i, (i = 1, 2, ..., C), & \\
    k_{ij} = q - \mathbf{h}_{i}^{T}\mathbf{h}_j - 2d, & \\
    k_{ij} \geq 0,
\end{array}
\right.
\end{split}
\end{equation}
which has both equality and inequality constraints; hence, we leverage the Augmented Lagrangian method~\cite{ALM} (ALM) and then re-formulate the optimization (\ref{optimization:objective_2}) as:
\begin{equation} \label{optimization:objective_3}
\begin{split}
\mathop{\min}_{\mathbf{H, M, K, \Lambda, \alpha}} & \text{L}  = ||\mathbf{S} - \frac{1}{q}\mathbf{H}^T\mathbf{M}||_{F}^{2} + \mu \sum_{i=1}^{C} \sum_{j \neq i} \mathbf{h}_{i}^{T}\mathbf{h}_j \\
&+ \sum_{i=1}^{C} \mathbf{\lambda}_{i}^{T}(\mathbf{h}_i - \mathbf{m}_i) + \sum_{i=1}^{C} \frac{\mathbf{\rho}_i}{2}||\mathbf{h}_i - \mathbf{m}_i||_{2}^{2} \\
&+\sum_{j \neq i} \alpha_{ij}(q - 2d - \mathbf{h}_{i}^{T}\mathbf{h}_j - \mathbf{k}_{ij}) \\
&+ \sum_{j \neq i} \frac{\beta_{ij}}{2}||q - 2d - \mathbf{h}_{i}^{T}\mathbf{h}_j - \mathbf{k}_{ij}||^2 \\
&\text{s.t.}
\left \{
\begin{array}{lr}
    \mathbf{h}_i \in \{-1, +1\}^q, & \\
    \mathbf{k}_{ij} \geq 0,
\end{array}
\right.
\end{split}
\end{equation}
where $\mathbf{K}$=$(\mathbf{k}_{ij})^{C \times C}$, $\mathbf{\Lambda}=[\lambda_1,\lambda_2,\cdots,\lambda_C]$, and $\alpha$=$(\alpha_{ij})^{C \times C}$ are learnable parameters, while $\mu$, $\mathbf{\rho}=(\rho_1,\rho_2,\cdots,\rho_C)$, and $\mathbf{\beta}=(\beta_{ij})^{C \times C}$ are non-negative hyper-parameters.

To solve the optimization problem (\ref{optimization:objective_3}), we adopt the \emph{alternating optimization} strategy: we iteratively optimize the variables \{$\mathbf{H}$, $\mathbf{M}$, $\mathbf{K}$, $\mathbf{\Lambda}$,  $\alpha$\} one by one, each time with the other variables fixed, until the objective function $\text{L}$ converges.

\textbf{\underline{Update $\mathbf{m}_i$}}. To update $\mathbf{m}_i$ with the other variables fixed, the sub-problem of $\text{L}$ in Eq.~(\ref{optimization:objective_3}) is an un-constrained optimization. 
The gradient of $\text{L}$ w.r.t. $\mathbf{m}_i$ can be easily calculated as:
\begin{equation} 
\begin{split}
\frac{\partial{L(\mathbf{m}_i)}}{\partial{\mathbf{m}_i}} 
& = \left(\frac{2}{q^2}\mathbf{HH}^T\mathbf{m}_i - \frac{2}{q}\mathbf{H}\mathbf{s}_i\right) \\
& + (- \mathbf{\lambda}_i) + (\mathbf{\rho}_i \mathbf{m}_i - \rho_i \mathbf{h}_i).
\end{split}
\end{equation}

By setting this gradient to zero, we can update $\mathbf{m}_i$ by:
\begin{equation} \label{update m}
\begin{split}
\mathbf{m}_i = 
\left(\frac{2}{q^2} \mathbf{HH}^T + \rho_i \mathbf{I}\right)^{-1}\left(\frac{2}{q}\mathbf{H}\mathbf{s}_i + \mathbf{\lambda}_i + \rho_i \mathbf{h}_i\right) ,
\end{split}
\end{equation}
which can be written in matrix forms, i.e.,
\begin{equation} \label{update M}
\begin{split}
\mathbf{M} = 
\left(\frac{2}{q^2} \mathbf{H}\mathbf{H}^T + \mathbf{\rho I}\right)^{-1}\left(\frac{2}{q}\mathbf{HS + \Lambda + \rho h_i}\right) ,
\end{split}
\end{equation}
where $\mathbf{I}$ denotes the identity matrix and $\rho_i = \rho$.

\textbf{\underline{Update $\mathbf{k}_{ij}$}}. To update $\mathbf{k}_{ij}$ with the other variables fixed, the sub-problem of $\text{L}$ in Eq.~(\ref{optimization:objective_3}) can be re-written as:
\begin{equation} 
\begin{split}
\mathop{\min}_{\mathbf{K}} \text{L} & = \alpha_{ij}(q - 2d - \mathbf{h}_{i}^{T}\mathbf{h}_j - \mathbf{k}_{ij}) \\
&+ \frac{\beta_{ij}}{2}||q - 2d - \mathbf{h}_{i}^{T}\mathbf{h}_j - \mathbf{k}_{ij}||_{2}^{2} \\
&\text{s.t.}
\begin{array}{lr}
    \mathbf{k}_{ij} \geq 0, (i \neq j) .
\end{array}
\end{split}
\end{equation}

The gradient of $\text{L}$ w.r.t. $\mathbf{k}_{ij}$ can be computed as:
\begin{equation} 
\frac{\partial{\text{L}(\mathbf{k}_{ij})}}{\partial{\mathbf{k}_{ij}}} = \beta_{ij}[(\mathbf{k}_{ij} - (q - 2d - \mathbf{h_{i}}^{T}\mathbf{h}_j)] - \alpha_{ij} .
\end{equation}

By setting this gradient to zero, and taking its inequality constraint into consideration, we can update $\mathbf{k}_{ij}$ by:
\begin{equation} \label{update k}
\mathbf{k}_{ij} = \max{\{q - 2d - \mathbf{h}_{i}^{T}\mathbf{h}_j + \frac{\alpha_{ij}}{\beta_{ij}}, 0\}} ,
\end{equation}
which can also be written in matrix forms, i.e,
\begin{equation} \label{update K}
\mathbf{K} = \max{\{q - 2d - \mathbf{H}^T\mathbf{H} + \mathbf{\frac{\alpha}{\beta}}, 0\}} .
\end{equation}

\textbf{\underline{Update $\mathbf{h}_i$}}. To update $\mathbf{h}_{i}$ with the other variables fixed, the sub-problem of $\text{L}$ in Eq.~(\ref{optimization:objective_3}) can be re-formulated as:
\begin{equation} 
\begin{split}
\mathop{\min}_{\mathbf{h}_i} \text{L} & = \sum_{i=1}^{C}||\mathbf{s}_i - \frac{1}{q}\mathbf{M}^T\mathbf{h}_i||_{2}^{2} + \mu \sum_{i=1}^{C} \sum_{j \neq i} \mathbf{h}_{i}^{T}\mathbf{h}_j \\
&+ \sum_{i=1}^{C} \mathbf{\lambda}_{i}^{T}(\mathbf{h}_i - \mathbf{m}_i) + \sum_{i=1}^{C} \frac{\mathbf{\rho}_i}{2}||\mathbf{h}_i - \mathbf{m}_i||_{2}^{2} \\
&+ \sum_{i=1}^{C} \sum_{j \neq i} \alpha_{ij}(q - 2d - \mathbf{h}_{i}^{T}\mathbf{h}_j - \mathbf{k}_{ij}) \\
&+ \sum_{i=1}^{C} \sum_{j \neq i} \frac{\beta_{ij}}{2}||q - 2d - \mathbf{h}_{i}^{T}\mathbf{h}_j - \mathbf{k}_{ij}||_{2}^{2} \\
& \text{s.t.}
\begin{array}{lr}
    \mathbf{h}_i \in \{-1, +1\}^q .
\end{array}
\end{split}
\end{equation}
The gradient of $\text{L}$ w.r.t. $\mathbf{h}_i$ can be achieved by:
\begin{equation} 
\begin{split}
\frac{\partial{\text{L}(\mathbf{h}_i)}}{\partial{\mathbf{h}_i}} 
& = \left(\frac{2}{q^2}\mathbf{MM}^T\mathbf{h}_i - \frac{2}{q}\mathbf{M}\mathbf{s}_i\right)  + (2 \mu \sum_{j = 1, j \neq i}^{C} \mathbf{h_j}) \\
& + (\mathbf{\lambda}_{i}^{T}) + \mathbf{\rho}_i(\mathbf{h}_i - \mathbf{m}_i)  + (-2 \mu \sum_{j = 1, j \neq i}^{C} \alpha_{ij}\mathbf{h}_j) \\
& + \sum_{j = 1, j \neq i}^{C} \beta_{ij}[2\mathbf{h}_j\mathbf{h}_{j}^{T}\mathbf{h}_i - 2(q -2d - \mathbf{k}_{ij}\mathbf{h}_j)] \\
& \text{s.t.}
\begin{array}{lr}
    \mathbf{h}_i \in \{-1, +1\}^q .
\end{array}
\end{split}
\end{equation}
Here, we utilize projected gradient descent (PGD)~\cite{PGD-JianFeng-SIAM-2018,DecentralizedPGD-Woocheol-CoRR-2023} to update $\mathbf{h}_i$ via:
\begin{equation} \label{update h}
\begin{split}
\mathbf{h}_i = \text{sign}\left[\mathbf{h}_i - \frac{1}{\eta} \frac{\partial L(\mathbf{h}_i)}{\partial \mathbf{h}_i}\right] .
\end{split}
\end{equation}
where $\eta$ is a hyper-parameter.

\textbf{\underline{Update $\mathbf{\lambda}_i$}}. 
To update $\mathbf{\lambda}_{i}$ with the other variables fixed, we can easily update $\mathbf{\lambda}_i$ via:
\begin{equation} \label{update lambda}
\mathbf{\lambda}_i = \mathbf{\lambda}_i + \mathbf{\rho}_i(\mathbf{h}_i - \mathbf{m}_i).
\end{equation}

\textbf{\underline{Update $\mathbf{\alpha}_{ij}$}}. To update $\alpha_{ij}$ with the other variables fixed, we can easily update $\alpha_{ij}$ with:
\begin{equation} \label{update alpha}
\alpha_{ij} = \alpha_{ij} +\beta_{ij} (q - 2d - \mathbf{h}_{i}^{T}\mathbf{h}_j - \mathbf{k}_{ij}).
\end{equation}

The above designed optimization procedure is summarized in Algorithm~\ref{alg:algorithm_1}.

\begin{algorithm}[tb]
    \caption{Semantic Hash Centers Generation}
    \label{alg:algorithm_1}
    \KwIn{Similarity matrix: $\mathbf{S}$; Code length: $q$; Number of classes: $C$; Minimal distance: $d$; Number of Training cycles: $T$; Hyper-parameters: $\mu, \rho, \eta, \beta$.}
    \KwOut{Semantic Hash Centers: $\mathbf{H}=[\mathbf{h}_1,\cdots,\mathbf{h}_C]$.}
    \text{\color{blue}{/*H has a minimum distance $d$*/}}\;
    Initialize hash centers $\{\mathbf{h}_1,\cdots,\mathbf{h}_C\}$ by MDS~\cite{MDS}\;
    \While{\textnormal{$t \leq T$}}{
        \text{\color{blue}{/*M is a proxy of H*/}}\;
        
        update $\mathbf{M}$ via Eq.~(\ref{update M})\;
        
        \text{\color{blue}{/*K is used to constrain the minimal distance*/}}\;
        update $\mathbf{K}$ via Eq.~(\ref{update K})\;
        
        \text{\color{blue}{/*column by column*/}}\;
        \While{\textnormal{$i \leq C$}}
        {
            \text{\color{blue}{/*-----$\mathbf{h}_i$-----*/}}\;
            \While{\textnormal{$l \leq 3$}}{
                update $\mathbf{h}_i$ via Eq.~(\ref{update h})\;
                $l \leftarrow{} l + 1$\;
            }
            \text{\color{blue}{/*-----$\mathbf{\lambda}_i$-----*/}}\;
            update $\lambda_i$ via Eq.~(\ref{update lambda})\;
            \text{\color{blue}{/*-----$\mathbf{\alpha}_{ij}$-----*/}}\;
            update $\alpha_{ij}$ via Eq.~(\ref{update alpha})\;
            $i \leftarrow{} i + 1$\;
        }
        $t \leftarrow t + 1$\;
    }
    Return semantic hash centers $\mathbf{H}=[\mathbf{h}_1,\cdots,\mathbf{h}_C]$.
\end{algorithm}

\subsection{Stage 3: Train the Deep Hashing Network}

With the generated $C$ semantic hash centers $\{\mathbf{h}_c\}_{c=1}^{C}$ as the supervised binary codes for images, we train a deep hashing network that can convert image $\mathbf{x}_i$ to its hash code $\mathbf{b}_i$ ($i=1,2,\cdots,N$). 
As shown in Fig.~\ref{fig:architecture} (Stage 3), the deep hashing network has three components.

The first component is the CNN backbone, which is used to capture the deep representations of the images. 
Here, we use ResNet34~\cite{ResNet-Kaiming-CVPR-2016} as the backbone, which is composed of a series of convolution layers and residual connections.

The second component is the hash layer, which is a fully-connected layer with the $\text{tanh}(\cdot)$ activation function that can map the image representation of $\mathbf{x}_i$ extracted by ResNet34 into a $q$-bit-like hash code $\mathbf{b}_{0i}$, and its final hash code $\mathbf{b}_i$ can be achieved via:
\begin{equation}
\mathbf{b}_i = \text{sign}(\mathbf{b}_{0i}) \in \{-1, +1\}^q.
\end{equation} 
Here, in the hash layer, we replace the $\text{sign}(\cdot)$ function with the $\text{tanh}(\cdot)$ function because the former is non-differentiable, making it unsuitable for training neural networks via backpropagation~\cite{BP-David-Nature-1986}; while the latter is differentiable and meanwhile can well approximate the former. 
Nevertheless, such an approximation can lead to quantization errors; hence, in the subsequent loss function, we will incorporate a quantization loss to compensate for these errors.

The third component is the loss function, and we utilize CSQ's~\cite{CSQ} loss function, which consists of the central similarity loss and the quantization loss. 
Since hash centers are binary vectors, we employ Binary Cross Entropy (BCE)~\cite{BCE} to replace the Hamming distance between hash code and its corresponding hash center, i.e.,
\begin{equation}
\text{D}_{Hamming}(\mathbf{h}_i, \mathbf{b}_{0i}) \propto \text{BCE}\left(\frac{\mathbf{h}_i+1}{2}, \frac{\mathbf{b}_{0i}+1}{2}\right) ,
\end{equation}
which indicates that the much closer the Hamming distance between a hash code and its corresponding hash center, the much smaller the value of the BCE function. Therefore, we can obtain the optimization objective of the central similarity loss $\text{L}_C$:
\begin{equation} 
\text{L}_C = 
	\frac{1}{N} \sum_{\substack{i = 1} }^{N} 
		\left(\frac{1+\mathbf{h}_i}{2}\text{log}\frac{1+\mathbf{b}_{0i}}{2} + \frac{1 - \mathbf{h}_{i}}{2}\text{log}\frac{1 - \mathbf{b}_{0i}}{2}\right) .
\end{equation}

Besides, the output binary-like vectors should be as close as possible to binary codes, thus we introduce a quantization loss $\text{L}_Q$ to refine the generated approximate hash codes $\mathbf{b}_{0i}$. 
Similar to DHN \cite{DHN}, we define:
\begin{equation} 
\text{L}_Q = \sum_{i = 1}^{N} |||\mathbf{b}_{0i}| - \mathbf{1}||_{2}^{2} ,
\end{equation}
where $\mathbf{1} \in \mathbb{R}^q$ is an all-one vector. 

Finally, we can summarize the overall central similarity optimization problem as:
\begin{equation}\label{overall-objective} 
\min_{\mathbf{\Theta}} \text{L} = \text{L}_C + \gamma \text{L}_Q ,
\end{equation}
where $\mathbf{\Theta}$ denotes all the trainable parameters of deep hashing network, and $\gamma$ is the hyper-parameter.

\section{Experiment}
This section will display the effectiveness of our proposed SHC approach via a series of experiments and ablation studies.

\subsection{Datasets}
We conduct experiments on three widely-used datasets, i.e., CIFAR-100\footnote{\url{https://www.cs.toronto.edu/~kriz/cifar.html}}~\cite{CIFAR-100}, 
Stanford Cars\footnote{\url{https://datasets.activeloop.ai/docs/ml/datasets/stanford-cars-dataset/}}~\cite{Stanford-Cars}, and NABirds\footnote{\url{https://dl.allaboutbirds.org/nabirds}}~\cite{NABirds}, based on which we curate 5 versions with different settings, described as below.

\textbf{\underline{CIFAR-100}.} It contains 100 classes, and each class has 600 images. 
Based on this dataset, we randomly sample 10,000 images as the training set, 5,000 images as the query set, and the remaining 45,000 images as the retrieval database.

\textbf{\underline{Stanford Cars}.} It holds 196 classes with 16,185 images in total. 
In its official website, it is divided almost evenly into two parts, i.e., 8,144 training images and 8,041 query images, and the 8,144 training images are also used as the retrieval databse. This version is marked as \textbf{Stanford Cars-A}.

By contrast, we randomly select 5,880, 1,958, and 8347 images as the training set, the query set, and the retrieval database, respectively. 
This version is marked as \textbf{Stanford Cars-B}.

\textbf{\underline{NABirds}.} It holds 555 classes, and has been partitioned into 24,633 training images and 23,929 query images; similar to the \textbf{Stanford Cars-A}, the 24,633 training images are also treated as the retrieval database. 
This version is marked as \textbf{NABirds-A}.

While we randomly choose 16,384 images for training, 8,171 images for queries, and 25,612 images as the retrieval databse. 
This version is marked as \textbf{NABirds-B}.

It's noteworthy that the version-B's training samples are much smaller than the version-A's, which implies that the learning on version-B's training set is much more challenging than that on version-A's; 
besides, the training set, the query set, and the retrieval database of version-B do not share any images, which is much more general and practical in real-world scenarios. 
The specific statistics of these five curated datasets are collected in Table~\ref{tab:Dataset-Statistics}.  

\begin{table}[tb]
	\centering
	\caption{Dataset Statistics.}
	\label{tab:Dataset-Statistics}
	\begin{tabular}{lrrr}
		\toprule
		\multicolumn{1}{c}{\textbf{Datasets}} & \multicolumn{1}{c}{\textbf{\#Training}} & \multicolumn{1}{c}{\textbf{\#Query}} & \multicolumn{1}{c}{\textbf{\#RetrievalDB}} \\ \midrule
		CIFAR-100                                & 10,000                                  & 5,000                                & 45,000                                            \\ 
		Stanford Cars-A                         & 8,184                                   & 8,041                                & 8,184                                             \\ 
		Stanford Cars-B                         & 5,880                                   & 1,958                                & 8,347                                             \\ 
		NABirds-A                               & 24,633                                  & 23,929                               & 24,633                                            \\ 
		NABirds-B                               & 16,384                                  & 8,171                                & 25,612                                            \\ \bottomrule
	\end{tabular}
\end{table}

\subsection{Evaluation Measures}
\label{Metrics}
To evaluate different methods' retrieval performances, we use four widely-used measures, i.e., Mean Average Precision (MAP@topK), Precision@topK curves, Recall@topK curves, and Precision-Recall curves~\cite{HashNet,EDMH-Yong-TKDE-2022,CSQ,OrthoHash,SCDH-Yong-TIP-2020}. 
For MAPs, we set topK to 100, 1000, and ALL (i.e., the total number of images in the retrieval database).
In terms of Precision/Recall curves, its topK is pre-defined as a large range from 1, 5, 10, $\cdots$, to 500.

Precision@topK refers to the proportion of images relevant to the query in the topK retrieved results. 
MAP@topK is a metric based on Precision@topK over all query images, also employed to evaluate a method's retrieval precisions.

Recall@topK refers to the ratio of successfully retrieved relevant images in the topK results among all the relevant images in the database. 

These two types of measures are often adopted to evaluate the performance of a retrieval system. 
However, there is a trade-off effect between the two kinds of metrics. 
Generally, as topK increases, precision tends to decrease while recall tends to rise.
Thus, the precision-recall curves present an overall illustration, taking both precision and recall into considerations, to assess a method's retrieval performance.



\subsection{Baseline Approaches}
We compare our SHC approach with several state-of-the-art image hashing methods, including:
\begin{itemize}
    \item \textbf{DPSH}\footnote{\url{https://github.com/TreezzZ/DPSH_PyTorch}}~\cite{DPSH} is the first deep network that can
    perform simultaneous feature and hash-code learning for retrievals with pairwise labels.

    \item \textbf{DTSH}\footnote{\url{https://github.com/Minione/DTSH}}~\cite{DTSH} puts forward a triplet based deep hashing method that aims to maximize the likelihood of the given triplet labels.

    \item \textbf{HashNet}\footnote{\url{https://github.com/thuml/HashNet}}~\cite{HashNet} learns binary hash codes by minimizing a novel weighted pairwise cross-entropy loss in deep CNN architectures. 

    \item \textbf{GreedyHash}\footnote{\url{https://github.com/ssppp/GreedyHash}}~\cite{GreedyHash}  adopts a greedy
    principle to tackle the discrete constrained models by iteratively updating the network toward the probable optimal discrete solution in each iteration.

    \item \textbf{IDHN}\footnote{\url{https://github.com/pectinid16/IDHN}}~\cite{IDHN} develops a quantified similarity formula to assess the fine-grained pairwise similarity for supervising hash-code generations.

    \item \textbf{LTH}\footnote{\url{https://github.com/butterfly-chinese/long-tail-hashing}}~\cite{LTH} addresses the problem of learning to hash for more realistic datasets where the labels of given datasets roughly exhibit a long-tail distribution.

    \item \textbf{CSQ}\footnote{\url{https://github.com/yuanli2333/Hadamard-Matrix-for-hashing}}~\cite{CSQ} proposes a new concept ``Hash Center'', and then learns hash codes by optimizing the Hamming distance between hash codes and corresponding hash centers.

    \item \textbf{MDS}\footnote{\url{https://github.com/Wangld5/Center-Hashing}}~\cite{MDS} is an improved variant of CSQ, addressing the issue of how to enforce a minimum distance between hash centers when it is not feasible to utilize Hadamard matrices to generate hash centers due to a large number of categories.

    \item \textbf{OrthoHash}\footnote{\url{https://github.com/kamwoh/orthohash}}~\cite{OrthoHash} unifies training objectives of deep hashing by maximizing the cosine similarity between the continuous codes and
    binary orthogonal target under a cross entropy loss.
\end{itemize}

Among the 9 deep hashing baselines, there are 3 methods using hash centers (i.e., OrthoHash\cite{OrthoHash}, CSQ\cite{CSQ}, and MDS\cite{MDS}), and 6 other deep hashing methods (i.e., LTH\cite{LTH}, IDHN\cite{IDHN}, GreedyHash\cite{GreedyHash}, HashNet\cite{HashNet}, DTSH\cite{DTSH} and DPSH\cite{DPSH}) using pointwise classification, pairwise or listwise semantic supervisions.

\subsection{Experimental Settings}
In stage 1, we utilize a pre-trained ResNet34\footnote{\url{https://pytorch.org/vision/0.12/generated/torchvision.models.resnet34.html}}~\cite{ResNet34} to extract image features, followed by a fully-connected layer to project the image features to the class space. 
The softmax function is then applied to get the probability vectors for classification. 
During training, we adopt the RMS optimizer with a cosine annealing schedule with an initial learning rate of 7e-5 for CIFAR-100 and 1e-4 for other datasets, and set the batch size and epochs to 64 and 300, respectively.

In stage 2, we first adopt MDS~\cite{MDS} for initializing hash centers and then execute Algorithm~\ref{alg:algorithm_1} for generating the semantic hash centers. 
With respect to parameters $\mathbf{\alpha}$ and $\mathbf{\lambda}$, we set their initial values to $\mathbf{0}^{C \times C}$ and $\mathbf{0.1}^{q \times C}$, respectively. 
For other parameters, we set $\rho = 0.2$, $\mu = 0.625$, $\mathbf{\beta}=\mathbf{(1e-6)}^{C \times C}$, and $\eta = 0.5$. 
Besides, we set the training cycles $T$ to 20, and perform 3 inner iterations to ensure convergence when optimizing the parameter $\mathbf{H}$.

In stage 3, we leverage a pre-trained ResNet34 to extract image features, followed by a FC layer with the $\text{tanh}(\cdot)$ activation function to approximate the binary output. 
During training, we use the RMS optimizer with a cosine annealing learning rate schedule. For the CIFAR-100 dataset, the initial learning rate is set to 7e-5, while for other datasets, it is set to 1e-4. 
Besides, we set the batch size, epochs, and $\gamma$ (Eq.~(\ref{overall-objective})) to 64, 300, and 1e-4, respectively.

To ensure a fair comparison across different approaches in the experiments, we employ the pre-trained ResNet-34 as the backbone for all competitors. 
For the selected baseline approaches, their parameters are all set in accordance with their corresponding papers.

All the curated datasets and codes will be released on Github.

\subsection{Results and Analysis}
This part will exhibit experimental results and some analysis stage by stage.

\subsubsection{Accuracies of Pre-trained Classification Network}
To yield semantic hash centers, we first need to pre-train a classification network, which is a ResNet34 based classifier trained via the cross-entropy loss; after training, its accuracies on the testing sets of various datasets are collected in Table~\ref{tab:pretrain-classification}.

\begin{table}[tb]
    \centering
    \caption{Accuracies of the pre-trained classification network on various datasets.}
    \label{tab:pretrain-classification}
    \setlength{\tabcolsep}{7mm}
    {
        \begin{tabular}{lc}
        \toprule
        \multicolumn{1}{c}{\textbf{Datasets}} & \textbf{Accuracy} \\
        \midrule
        CIFAR-100 & 66.72\% \\
        Stanford Cars-A & 73.49\% \\
        Stanford Cars-B & 62.56\% \\
        NABirds-A & 52.14\% \\
        NABirds-B & 48.44\% \\
        \bottomrule
        \end{tabular}
    }
\end{table}

From Table~\ref{tab:pretrain-classification}, we can see that the classification accuracies on all datasets are not very high, but it does not affect our method to achieve good retrieval performance (see below). 
Although many images are classified into the wrong classes, we only focus on the mean (center) of all images within each class, and hence these classfication errors can be diluted out under the most correctly-classified images.

\subsubsection{Effectiveness of Generated Semantic Hash Centers}

By conducting the optimization Algorithm~\ref{alg:algorithm_1} in Section~\ref{stage2}, we can obtain semantic hash centers; then we compare our method with three other approaches, i.e., random-generated hash centers~\cite{CSQ}, hash centers generated using Hadamard matrices~\cite{CSQ}, and hash centers generated using the Gilbert-Varshamov bound~\cite{MDS}. 
We compare these methods in terms of the minimal Hamming distance ($d_{min}$) between any two hash centers, and the pairwise similarity losses ($S_{loss}$) of the generated hash centers, which are formulated as follows: 
\begin{align}
    d_{min} & = \min\{\frac{q - \mathbf{H}^T\mathbf{H}\odot(\mathbf{1} - \mathbf{I})}{2}\}, \\
    S_{loss} & =\text{L}_{semantic} =||\mathbf{S} - \frac{1}{q}\mathbf{H}^T\mathbf{H}||^{2}_{F};
\end{align}
and obviously a smaller $S_{loss}$ indicates a better preservation of semantics between classes.

Table~\ref{tab:min_d-S_loss} displays the minimal distances and the pairwise similarity losses of the generated hash centers w.r.t. different methods. 
Evidently, our method achieves the largest $d_{min}$ and the smallest $S_{loss}$ across all datasets, which validates the effectiveness of our hash centers generation algorithm, i.e., not only well keeping semantics preserved but also setting hash centers apart as far as possible.

\begin{table*}[!ht]
    \centering
    \caption{Comparisons of different hash center generation methods w.r.t. the minimal distance $d_{min}$ and the similarity loss $S_{loss}$.}
    \label{tab:min_d-S_loss}
    \setlength{\tabcolsep}{4.4mm}
    {
        \begin{tabular}{c | c | cc | cc | cc}
        \toprule
            \multirow{2}{*}{\textbf{Datasets}} & \multirow{2}{*}{\textbf{Methods}} & \multicolumn{2}{c|}{\textbf{16 bits}} & \multicolumn{2}{c|}{\textbf{32 bits}} & \multicolumn{2}{c}{\textbf{64 bits}} \\ 
            \cmidrule(lr{.75em}){3-4} \cmidrule(lr{.75em}){5-6} \cmidrule(lr{.75em}){7-8} 
            & & $d_{min}$ & $S_{loss}$ & $d_{min}$ & $S_{loss}$ & $d_{min}$ & $S_{loss}$ \\
            \cmidrule(lr{.75em}){1-1} \cmidrule(lr{.75em}){2-2} \cmidrule(lr{.75em}){3-4} \cmidrule(lr{.75em}){5-6} \cmidrule(lr{.75em}){7-8} 
    
            \multirow{4}{*}{\makecell[c]{CIFAR-100}} 
            & random & 0 & 0.1094 & 5 & 0.0767 & 17 & 0.0612 \\
            & CSQ & 0 & 0.1074 & 4 & 0.0697 & 32 & 0.0528 \\
            & MDS & 4 & 0.0993 & 10 & 0.0698 & 32 & 0.0528 \\
            & SHC & 4 & \textbf{0.0545} & 10 & \textbf{0.0268} & 24 & \textbf{0.0321} \\ 
            \cmidrule(lr{.75em}){1-1} \cmidrule(lr{.75em}){2-2} \cmidrule(lr{.75em}){3-4} \cmidrule(lr{.75em}){5-6} \cmidrule(lr{.75em}){7-8}
            
            \multirow{4}{*}{\makecell[c]{Stanford Cars-A}} 
            & random & 0 & 0.0876 & 3 & 0.0569 & 14 & 0.0413 \\
            & CSQ & 0 & 0.0907 & 4 & 0.0558 & 14 & 0.0379 \\
            & MDS & 4 & 0.0831 & 10 & 0.0534 & 23 & 0.0375 \\
            & SHC & 4 & \textbf{0.0623} & 10 & \textbf{0.0308} & 23 & \textbf{0.0205} \\ 
            \cmidrule(lr{.75em}){1-1} \cmidrule(lr{.75em}){2-2} \cmidrule(lr{.75em}){3-4} \cmidrule(lr{.75em}){5-6} \cmidrule(lr{.75em}){7-8}
            
            \multirow{4}{*}{\makecell[c]{Stanford Cars-B}} 
            & random & 0 & 0.0909 & 3 & 0.0592 & 14 & 0.0444 \\
            & CSQ & 0 & 0.0932 & 4 & 0.0583 & 14 & 0.0410 \\
            & MDS & 4 & 0.0861 & 10 & 0.0557 & 23 & 0.0410 \\
            & SHC & 4 & \textbf{0.0706} & 10 & \textbf{0.0301} & 23 & \textbf{0.0237} \\ 
            \cmidrule(lr{.75em}){1-1} \cmidrule(lr{.75em}){2-2} \cmidrule(lr{.75em}){3-4} \cmidrule(lr{.75em}){5-6} \cmidrule(lr{.75em}){7-8}
    
            \multirow{4}{*}{\makecell[c]{NABirds-A}} 
            & random & 0 & 0.0720 & 3 & 0.0406 & 13 & 0.0252 \\
            & CSQ & 0 & 0.0754 & 4 & 0.0410 & 16 & 0.0249 \\
            & MDS & 3 & 0.0702 & 9 & 0.0393 & 21 & 0.0242 \\
            & SHC & 3 & \textbf{0.0664} & 9 & \textbf{0.0316} & 21 & \textbf{0.0140} \\ 
            \cmidrule(lr{.75em}){1-1} \cmidrule(lr{.75em}){2-2} \cmidrule(lr{.75em}){3-4} \cmidrule(lr{.75em}){5-6} \cmidrule(lr{.75em}){7-8}
    
            \multirow{4}{*}{\makecell[c]{NABirds-B}} 
            & random & 0 & 0.0725 & 3 & 0.0413 & 13 & 0.0258 \\
            & CSQ & 0 & 0.0761 & 4 & 0.0417 & 16 & 0.0255 \\
            & MDS & 3 & 0.0710 & 9 & 0.0398 & 21 & 0.0248 \\
            & SHC & 3 & \textbf{0.0668} & 9 & \textbf{0.0315} & 21 & \textbf{0.0144} \\ 
              
            \bottomrule
        \end{tabular}
    }
\end{table*}

\subsubsection{Retrieval Performances of Different Approaches} \label{Retrieval Accuracies}

\begin{table*}[!h]
	\centering
	\caption{Different methods’ mAP@topK=100 values on CIFAR-100, Stanford Cars, and NABirds datasets (the code length is set to 16, 32 and 64 bits).}
	\label{tab:map@100}
	\setlength{\tabcolsep}{0.5mm}
	{
		\renewcommand\arraystretch{1.5}
        \small
		\begin{tabular}{c | ccc | ccc | ccc | ccc | ccc}
			\toprule
			\multirow{2}{*}{\textbf{Methods}} & \multicolumn{3}{c|}{\textbf{CIFAR-100}} & \multicolumn{3}{c|}{\textbf{Stanford Cars-A}}  & \multicolumn{3}{c|}{\textbf{Stanford Cars-B}} & \multicolumn{3}{c|}{\textbf{NABirds-A}} & \multicolumn{3}{c}{\textbf{NABirds-B}}   \\ \cmidrule(lr{.75em}){2-4} \cmidrule(lr{.75em}){5-7} \cmidrule(lr{.75em}){8-10} \cmidrule(lr{.75em}){11-13} \cmidrule(lr{.75em}){14-16}
			~ & 16 bits & 32 bits & 64 bits & 16 bits & 32 bits & 64 bits & 16 bits & 32 bits & 64 bits & 16 bits & 32 bits & 64 bits & 16 bits & 32 bits & 64 bits  \\ \cmidrule(lr{.75em}){1-1} \cmidrule(lr{.75em}){2-4} \cmidrule(lr{.75em}){5-7} \cmidrule(lr{.75em}){8-10} \cmidrule(lr{.75em}){11-13} \cmidrule(lr{.75em}){14-16}
			DPSH & 0.1572  & 0.3541  & 0.5054  & 0.0308  & 0.0533  & 0.1421  & 0.0574  & 0.0989  & 0.1264  & 0.0107  & 0.0150  & 0.0157  & 0.0117  & 0.0191  & 0.0403   \\
			DTSH & 0.4939  & 0.5416  & 0.5898  & 0.4236  & 0.5425  & 0.6589  & 0.3533  & 0.4781  & 0.5339  & 0.1237  & 0.1645  & 0.2422  & 0.0145  & 0.1487  & 0.1811   \\
			HashNet & 0.2085  & 0.4115  & 0.3914  & 0.1216  & 0.2641  & 0.2554  & 0.0644  & 0.2134  & 0.2134  & 0.0566  & 0.1737  & 0.1902  & 0.0385  & 0.1232  & 0.1661   \\
			GreedyHash & 0.5690  & 0.6013  & 0.5941  & \underline{0.5998}  & 0.6686  & 0.6962  & \underline{0.3966}  & \underline{0.4942}  & 0.5375  & 0.4687  & 0.5545  & 0.5886  & \underline{0.3264}  & 0.4115  & 0.4541   \\
			IDHN & 0.1127  & 0.2839  & 0.4883  & 0.0293  & 0.0433  & 0.3322  & 0.0444  & 0.1326  & 0.2728  & 0.0139  & 0.0137  & 0.0151  & 0.0148  & 0.0247  & 0.0401   \\
			CSQ & 0.6123  & \underline{0.6559}  & {0.6620}  & 0.5401  & \underline{0.6827}  & 0.7342  & 0.3219  & 0.4895  & \underline{0.5853}  & 0.4540  & 0.5699  & 0.6243  & 0.2956  & 0.4309  & 0.4984   \\
			MDS & \underline{0.6199}  & 0.6519  & \underline{0.6649}  & {0.5866}  & 0.6801  & \underline{0.7396}  & 0.3238  & 0.4804  & 0.5697  & \underline{0.4812}  & \underline{0.5791}  & \underline{0.6248}  & 0.3247  & \underline{0.4359}  & \underline{0.4990}   \\
			LTH & 0.5687  & 0.6171  & 0.6405  & 0.5753  & 0.6612  & 0.7157  & 0.3768  & 0.4915  & 0.5406  & 0.4461  & 0.5572  & 0.5926  & 0.2861  & 0.4004  & 0.4521   \\
			OrthoHash & 0.5709  & 0.5871  & 0.5866  & 0.4824  & 0.5726  & 0.6064  & 0.2762  & 0.3888  & 0.4496  & 0.4583  & 0.5152  & 0.5242  & 0.2721  & 0.3420  & 0.3501   \\ \cmidrule(lr{.75em}){1-1} \cmidrule(lr{.75em}){2-4} \cmidrule(lr{.75em}){5-7} \cmidrule(lr{.75em}){8-10} \cmidrule(lr{.75em}){11-13} \cmidrule(lr{.75em}){14-16}
			\textbf{SHC} & \textbf{0.6360} & \textbf{0.6629} & \textbf{0.6672} & \textbf{0.6674} & \textbf{0.7448} & \textbf{0.7776} & \textbf{0.4440} & \textbf{0.5859} & \textbf {0.6288} & \textbf{0.4986} & \textbf{0.6117} & \textbf{0.6586} & \textbf{0.3552} & \textbf{0.4753} & \textbf{0.5289}  \\
			\bottomrule
		\end{tabular}
	}
\end{table*}

\begin{table*}[!h]
	\centering
	\caption{Different methods’ mAP@topK=1000 values on CIFAR-100, Stanford Cars, and NABirds datasets (the code length is set to 16, 32 and 64 bits).}
	\label{tab:map@1000}
	\setlength{\tabcolsep}{0.5mm}
	{
		\renewcommand\arraystretch{1.5}
        \small
		\begin{tabular}{c | ccc | ccc | ccc | ccc | ccc}
			\toprule
			\multirow{2}{*}{\textbf{Methods}} & \multicolumn{3}{c|}{\textbf{CIFAR-100}} & \multicolumn{3}{c|}{\textbf{Stanford Cars-A}}  & \multicolumn{3}{c|}{\textbf{Stanford Cars-B}} & \multicolumn{3}{c|}{\textbf{NABirds-A}} & \multicolumn{3}{c}{\textbf{NABirds-B}}   \\ \cmidrule(lr{.75em}){2-4} \cmidrule(lr{.75em}){5-7} \cmidrule(lr{.75em}){8-10} \cmidrule(lr{.75em}){11-13} \cmidrule(lr{.75em}){14-16}
			~ & 16 bits & 32 bits & 64 bits & 16 bits & 32 bits & 64 bits & 16 bits & 32 bits & 64 bits & 16 bits & 32 bits & 64 bits & 16 bits & 32 bits & 64 bits  \\ \cmidrule(lr{.75em}){1-1} \cmidrule(lr{.75em}){2-4} \cmidrule(lr{.75em}){5-7} \cmidrule(lr{.75em}){8-10} \cmidrule(lr{.75em}){11-13} \cmidrule(lr{.75em}){14-16}
			DPSH & 0.1423  & 0.2991  & 0.4332  & 0.0262  & 0.0255  & 0.0809  & 0.0329  & 0.0483  & 0.0606  & 0.0124  & 0.0148  & 0.0143  & 0.0132  & 0.0179  & 0.0242   \\ 
			DTSH & 0.4260  & 0.4688  & 0.5420  & 0.3898  & 0.5158  & 0.6468  & 0.2789  & 0.4002  & 0.4509  & 0.0704  & 0.1032  & 0.1748  & 0.0519  & 0.0812  & 0.1057   \\ 
			HashNet & 0.2032  & 0.3582  & 0.3395  & 0.1039  & 0.2404  & 0.2186  & 0.0586  & 0.1572  & 0.1626  & 0.0582  & 0.1590  & 0.1357  & 0.0424  & 0.0950  & 0.1023   \\ 
			GreedyHash & 0.5230  & 0.5619  & 0.5574  & \underline{0.5858}  & 0.6598  & 0.6895  & \underline{0.3112}  & \underline{0.4104}  & 0.4659  & 0.4528  & 0.5463  & 0.5843  & \underline{0.2589}  & 0.3361  & 0.3914   \\ 
			IDHN & 0.0983  & 0.2378  & 0.4208  & 0.0201  & 0.0224  & 0.2990  & 0.0257  & 0.0884  & 0.2065  & 0.0145  & 0.0134  & 0.0143  & 0.0164  & 0.0197  & 0.0226   \\ 
			CSQ & 0.5686  & \underline{0.6113}  & 0.6117  & 0.5200  & \underline{0.6742}  & 0.7295  & 0.2406  & 0.4026  & \underline{0.5064}  & 0.4427  & 0.5633  & 0.6213  & 0.2279  & 0.3562  & 0.4297   \\ 
			MDS & \underline{0.5775}  & 0.6074  & \underline{0.6167}  & 0.5732  & 0.6722  & \underline{0.7358}  & 0.2426  & 0.3961  & 0.4959  & \underline{0.4736}  & \underline{0.5743}  & \underline{0.6234}  & 0.2536  & \underline{0.3640}  & \underline{0.4372}   \\ 
			LTH & 0.5255  & 0.5718  & 0.5859  & 0.5568  & 0.6485  & 0.7093  & {0.2882}  & 0.4012  & 0.4589  & 0.4248  & 0.5464  & 0.5884  & 0.2166  & 0.3195  & 0.3715   \\ 
			OrthoHash & 0.4869  & 0.4955  & 0.5028  & 0.4327  & 0.5241  & 0.5563  & 0.1908  & 0.2885  & 0.3426  & 0.4215  & 0.4785  & 0.4840  & 0.1865  & 0.2494  & 0.2543   \\
			\cmidrule(lr{.75em}){1-1} \cmidrule(lr{.75em}){2-4} \cmidrule(lr{.75em}){5-7} \cmidrule(lr{.75em}){8-10} \cmidrule(lr{.75em}){11-13} \cmidrule(lr{.75em}){14-16}
			\textbf{SHC} & \textbf{0.5936}  & \textbf{0.6166}  & \textbf{0.6170}  & \textbf{0.6525}  & \textbf{0.7379}  & \textbf{0.7762}  & \textbf{0.3541}  & \textbf{0.5033}  & \textbf{0.5573}  & \textbf{0.4878}  & \textbf{0.6072}  & \textbf{0.6608}  & \textbf{0.2813}  & \textbf{0.3991}  & \textbf{0.4575}   \\
			\bottomrule
		\end{tabular}
	}
\end{table*}

\begin{table*}[!h]
	\centering
	\caption{Different methods’ mAP@topK=ALL values on CIFAR-100, Stanford Cars, and NABirds datasets (the code length is set to 16, 32 and 64 bits).}
	\label{tab:map@ALL}
	\setlength{\tabcolsep}{0.5mm}
	{
		\renewcommand\arraystretch{1.5}
        \small
		\begin{tabular}{c | ccc | ccc | ccc | ccc | ccc}
			\toprule
			\multirow{2}{*}{\textbf{Methods}} & \multicolumn{3}{c|}{\textbf{CIFAR-100}} & \multicolumn{3}{c|}{\textbf{Stanford cars-A}}  & \multicolumn{3}{c|}{\textbf{Stanford cars-B}} & \multicolumn{3}{c|}{\textbf{NABirds-A}} & \multicolumn{3}{c}{\textbf{NABirds-B}}   \\ \cmidrule(lr{.75em}){2-4} \cmidrule(lr{.75em}){5-7} \cmidrule(lr{.75em}){8-10} \cmidrule(lr{.75em}){11-13} \cmidrule(lr{.75em}){14-16}
			~ & 16 bits & 32 bits & 64 bits & 16 bits & 32 bits & 64 bits & 16 bits & 32 bits & 64 bits & 16 bits & 32 bits & 64 bits & 16 bits & 32 bits & 64 bits  \\ \cmidrule(lr{.75em}){1-1} \cmidrule(lr{.75em}){2-4} \cmidrule(lr{.75em}){5-7} \cmidrule(lr{.75em}){8-10} \cmidrule(lr{.75em}){11-13} \cmidrule(lr{.75em}){14-16}
			DPSH & 0.0833  & 0.2403  & 0.3697  & 0.0128  & 0.0130  & 0.0587  & 0.0199  & 0.0300  & 0.0383  & 0.0059  & 0.0062  & 0.0059  & 0.0056  & 0.0061  & 0.0060   \\
			DTSH & 0.3497  & 0.3914  & 0.4513  & 0.3850  & 0.5120  & 0.6418  & \underline{0.2573}  & \underline{0.3792}  & 0.4281  & 0.0547  & 0.0918  & 0.1653  & 0.0318  & 0.0603  & 0.0843   \\
			HashNet & 0.1487  & 0.3041  & 0.2930  & 0.0996  & 0.2383  & 0.2174  & 0.0470  & 0.1441  & 0.1519  & 0.0562  & 0.1557  & 0.1347  & 0.0342  & 0.0754  & 0.0872   \\
			GreedyHash & 0.3809  & 0.4204  & 0.4238  & \underline{0.5764}  & 0.6497  & 0.6802  & 0.2406  & 0.3320  & 0.3929  & 0.4407  & 0.5330  & 0.5721  & \underline{0.1646}  & 0.2259  & 0.2824   \\
			IDHN & 0.0642  & 0.1797  & 0.3652  & 0.0094  & 0.0112  & 0.2907  & 0.0137  & 0.0698  & 0.1787  & 0.0077  & 0.0068  & 0.0062  & 0.0072  & 0.0077  & 0.0066   \\ 
			CSQ & 0.4272  & \underline{0.4875}  & 0.5185  & 0.5087  & \underline{0.6651}  & 0.7233  & 0.1785  & 0.3301  & \underline{0.4462}  & 0.4311  & 0.5535  & 0.6146  & 0.1387  & 0.2462  & 0.3244   \\ 
			MDS & \underline{0.4388}  & 0.4856  & \underline{0.5205}  & 0.5613  & 0.6597  & \underline{0.7273}  & 0.1837  & 0.3251  & 0.4313  & \underline{0.4627}  & \underline{0.5650}  & \underline{0.6155}  & 0.1597  & \underline{0.2542}  & \underline{0.3313}   \\ 
			LTH & 0.3890  & 0.4544  & 0.4980  & 0.5449  & 0.6395  & 0.7032  & 0.2221  & 0.3332  & 0.4065  & 0.4066  & 0.5350  & 0.5811  & 0.1334  & 0.2380  & 0.3135   \\ 
			OrthoHash & 0.3813  & 0.3948  & 0.4062  & 0.4225  & 0.5154  & 0.5491  & 0.1572  & 0.2537  & 0.3147  & 0.4121  & 0.4702  & 0.4753  & 0.1403  & 0.2069  & 0.2175   \\ \cmidrule(lr{.75em}){1-1} \cmidrule(lr{.75em}){2-4} \cmidrule(lr{.75em}){5-7} \cmidrule(lr{.75em}){8-10} \cmidrule(lr{.75em}){11-13} \cmidrule(lr{.75em}){14-16}
			\textbf{SHC} & \textbf{0.4838}  & \textbf{0.5186}  & \textbf{0.5262}  & \textbf{0.6433}  & \textbf{0.7300}  & \textbf{0.7720}  & \textbf{0.2976}  & \textbf{0.4547}  & \textbf{0.5139}  & \textbf{0.4755}  & \textbf{0.6001}  & \textbf{0.6566}  & \textbf{0.1926}  & \textbf{0.3185}  & \textbf{0.3988}   \\
			\bottomrule
		\end{tabular}
	}
\end{table*}

Table \ref{tab:map@100}, \ref{tab:map@1000}, and \ref{tab:map@ALL} present the MAP@topK (topK=100, 1000, and ALL) values of our SHC approach and 9 baselines on 5 datasets with 3 different hash code lengths (i.e., 16, 32, and 64 bits). 
In all the tables, the best results are highlighted in bold, while the second-best are underlined. 
Obviously, SHC performs better than other methods, achieving an improvement of \textbf{(+0.13 \% $\sim$ +11.53\%), (+5.14 \% $\sim$ +13.77\%), (+7.43 \% $\sim$ +22.64\%), (+2.77 \% $\sim$ +6.68\%), (+4.64 \% $\sim$ +25.30\%)} over the second-best competitor on CIFAR-100, Stanford Cars-A, Stanford Cars-B, NABirds-A, and NABirds-B datasets, respectively. 

To be specific, the improvements brought by our SHC on CIFAR-100 are not as significant as those on the other two datasets. 
This is attributed to the fact that the CIFAR-100 dataset has fewer categories, allowing for larger Hamming distances between hash centers. 
In other words, the Hamming space is still relatively ``spacious''. 
By the way, as the hash code length becomes larger, the Hamming space becomes much more ``spacious'', leading to a diminishing improvement.

On the Stanford Cars dataset, the largest improvement is observed with a hash code length of 32 bits. 
This could be because, at 16 bits, the Hamming space becomes too ``crowded'', leaving insufficient room for semantic adjustments by the hash centers, resulting in a decrease in improvement. 
On the other hand, when the hash code length is 64 bits, the Hamming space becomes too ``spacious'', similar to the CIFAR-100 dataset, which also leads to a decrease in improvement. 
However, a 32-bit Hamming space strikes the right balance for the Stanford Cars dataset with 196 categories, achieving the best optimization results.

In our experiments on the NABirds dataset, which has the largest number of categories, the improvements increase with the hash code length. 
This is consistent with the above analysis, where the 16-bit and 32-bit Hamming spaces are too ``crowded'' (555 classes); 
however, a 64-bit Hamming space allows for better semantic optimization by the hash centers without being overly ``spacious''.


Fig.~\ref{fig:precision-recall-curves} shows the retrieval performance w.r.t. Precision-Recall curves on five datasets. Our method exhibits the optimal PR curves on almost all datasets. On the CIFAR-100 dataset, as the length of the hash codes increases, the advantage of our method in terms of the PR curve gradually diminishes. On the Stanford Cars dataset, our method demonstrates the most significant advantage in the PR curve when the hash code length is 32 bits. On the NABirds dataset, the advantage of our method in the PR curve becomes more pronounced with the increase in hash code length. These observations align with the above analysis presented in Section~\ref{Retrieval Accuracies}.


Fig.~\ref{fig:precision-at-topK} shows the retrieval performance w.r.t. Precision-topK curves on five datasets. The Precision-topK curves of our method are optimal on almost all datasets, and the trend of the curves aligns with the theoretical expectations described in Section \ref{Metrics}, where precision gradually decreases as topK increases. However, in the case of the NABirds dataset with a hash code length of 16 bits, our method's precision curve falls slightly behind. This could be attributed to the fact that for the NABirds dataset with a large number of categories, a lower hash code length may not provide sufficient distance between hash centers, resulting in a decrease in the precision curve.


Fig.~\ref{fig:recall-at-topK} shows the retrieval performance w.r.t. Recall-topK curves on five datasets. On most datasets, the Recall-topK curves of our method are the best, and the trend of the curves aligns with the theoretical expectations described in Section \ref{Metrics}, where recall gradually increases as topK increases. However, on the Stanford Cars dataset, our method's Recall curves are not always optimal. This could be attributed to the composition of the dataset, where the training data we used is relatively small, and the differences between categories in the dataset may not be significant, which could have a negative impact on Recall.

In the case of the NABirds-B dataset with a hash code length of 16 bits, our method's Recall curve slightly falls behind. The reason could be similar to the precision curve, where a lower hash code length may not provide sufficient distance between hash centers for a dataset with a large number of categories like NABirds, resulting in a decrease in the Recall curve.

\begin{figure}[H]
	\centering\subfloat[16 bits, CIFAR-100]{\includegraphics[width=3.7cm]{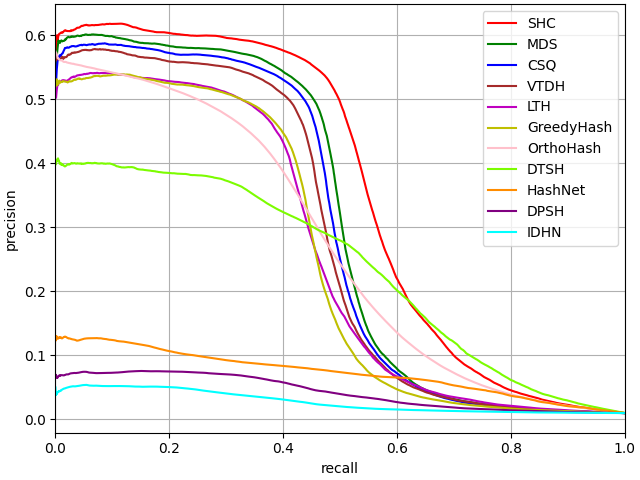}}\hspace{4mm}
	\centering\subfloat[32 bits, CIFAR-100]{\includegraphics[width=3.7cm]{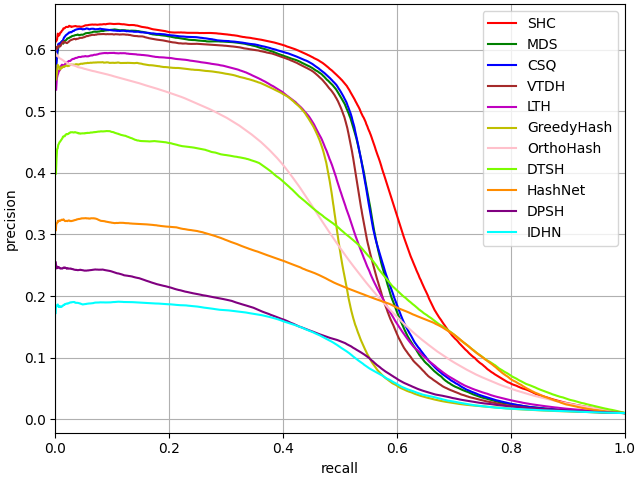}}\hspace{4mm}
	\centering\subfloat[64 bits, CIFAR-100]{\includegraphics[width=3.7cm]{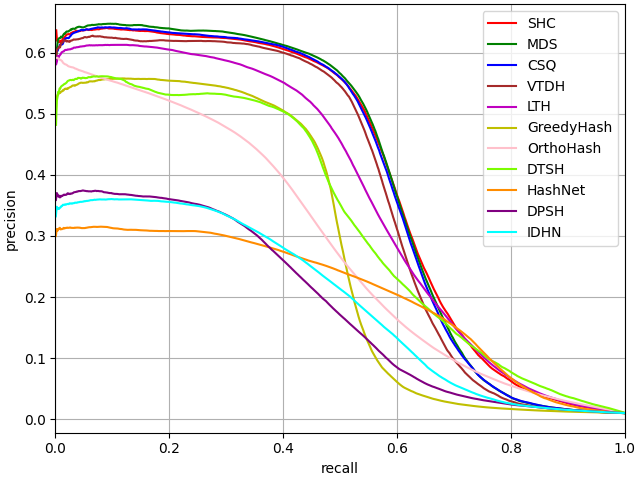}}
	
	\centering\subfloat[16 bits, Stanford Cars-A]{\includegraphics[width=3.7cm]{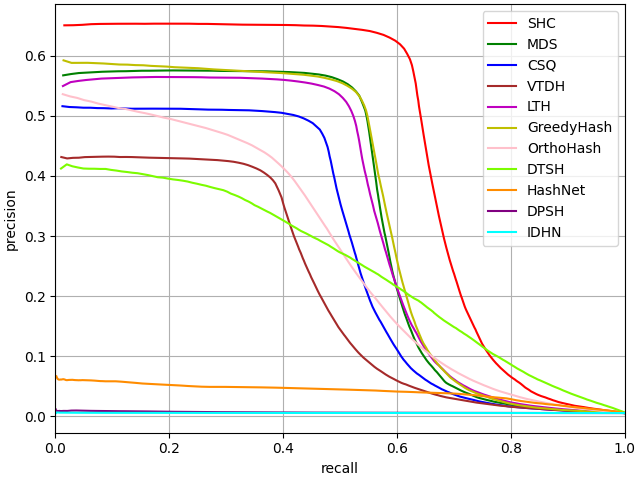}}\hspace{4mm}
	\centering\subfloat[32 bits, Stanford Cars-A]{\includegraphics[width=3.7cm]{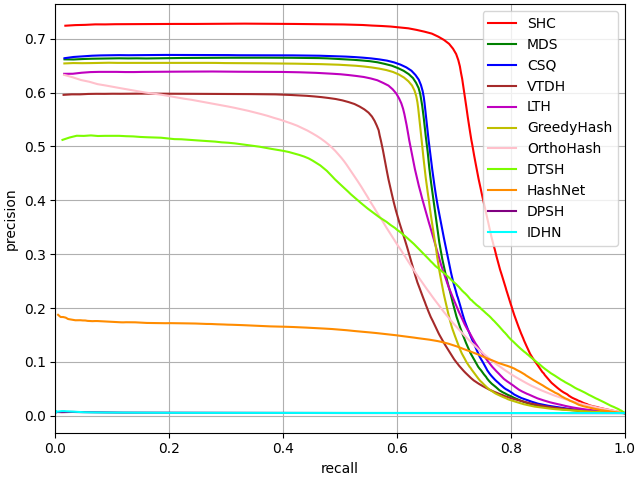}}\hspace{4mm}
	\centering\subfloat[64 bits, Stanford Cars-A]{\includegraphics[width=3.7cm]{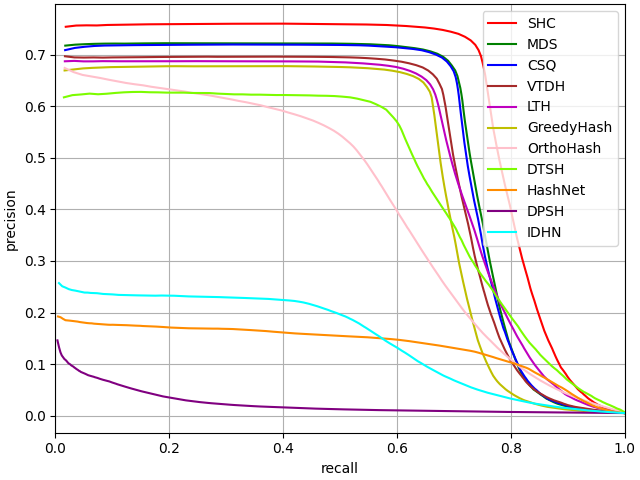}}
	
	\centering\subfloat[16 bits, Stanford Cars-B]{\includegraphics[width=3.7cm]{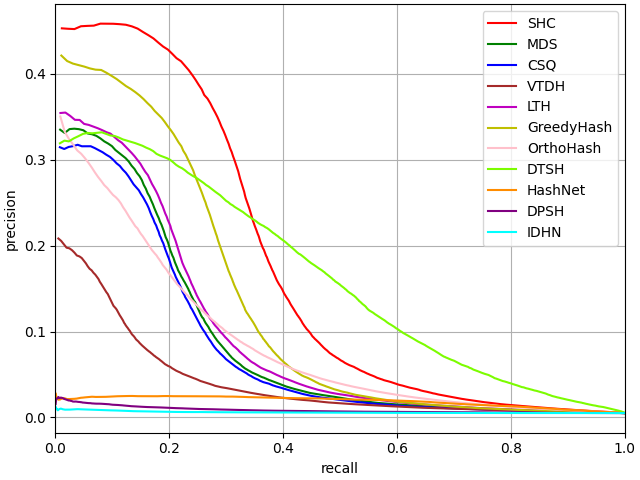}}\hspace{4mm}
	\centering\subfloat[32 bits, Stanford Cars-B]{\includegraphics[width=3.7cm]{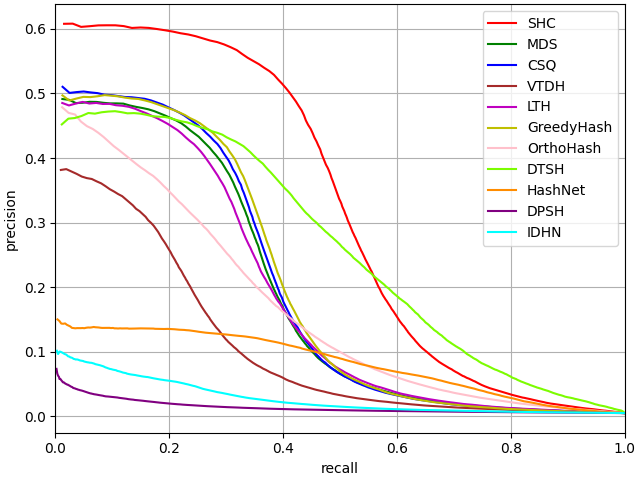}}\hspace{4mm}
	\centering\subfloat[64 bits, Stanford Cars-B]{\includegraphics[width=3.7cm]{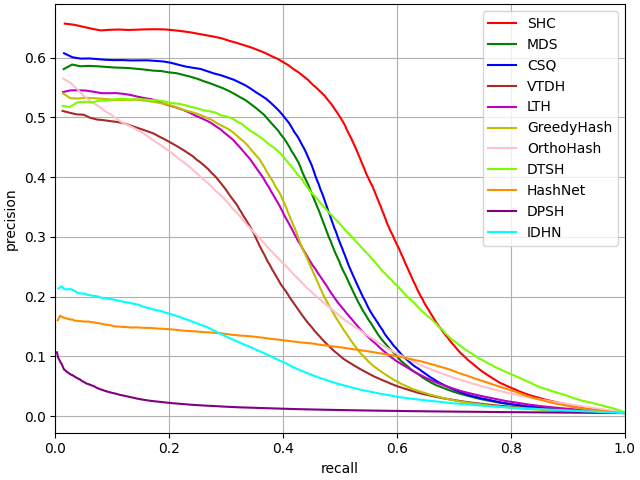}}
	
	\centering\subfloat[16 bits, NABirds-A]{\includegraphics[width=3.7cm]{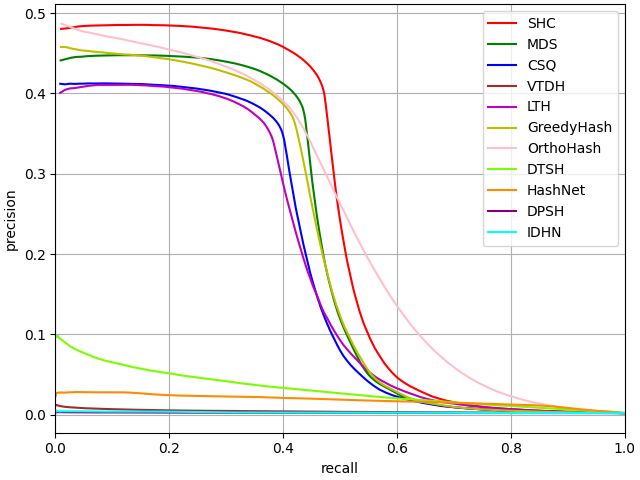}}\hspace{4mm}
	\centering\subfloat[32 bits, NABirds-A]{\includegraphics[width=3.7cm]{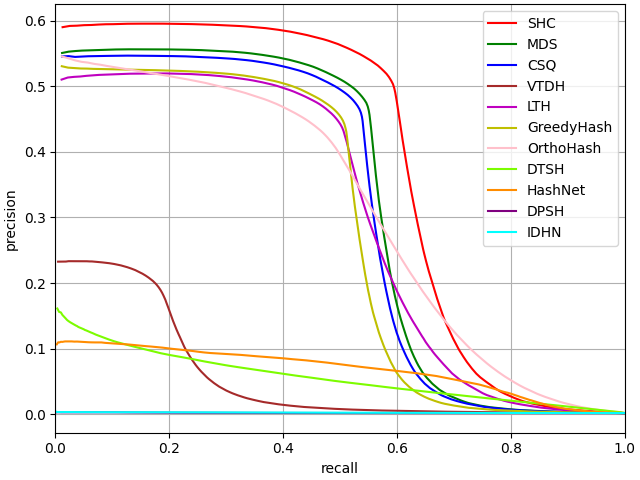}}\hspace{4mm}
	\centering\subfloat[64 bits, NABirds-A]{\includegraphics[width=3.7cm]{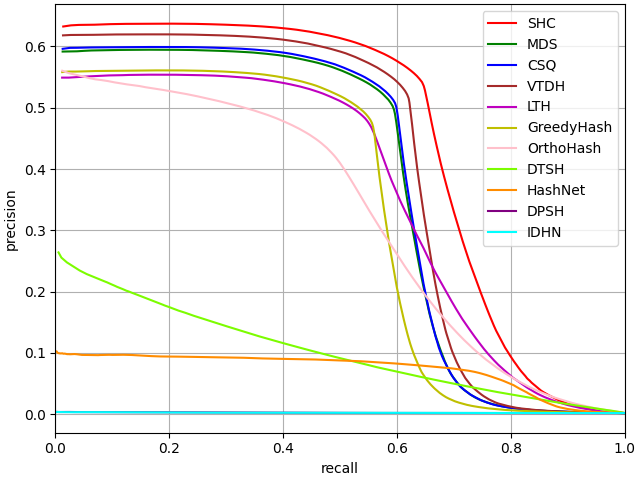}}
	
	\centering\subfloat[16 bits, NABirds-B]{\includegraphics[width=3.7cm]{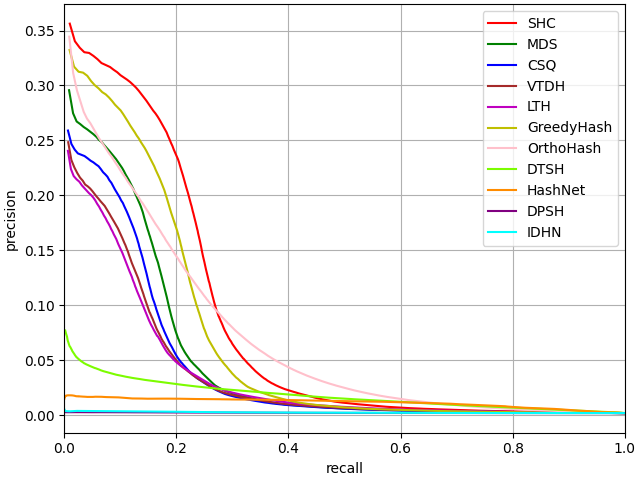}}\hspace{4mm}
	\centering\subfloat[32 bits, NABirds-B]{\includegraphics[width=3.7cm]{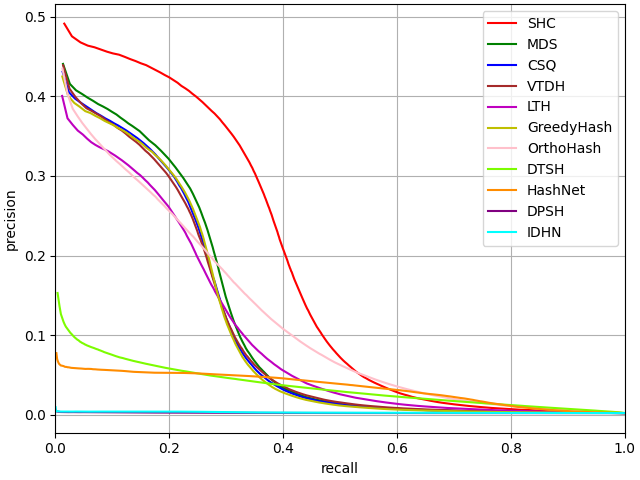}}\hspace{4mm}
	\centering\subfloat[64 bits, NABirds-B]{\includegraphics[width=3.7cm]{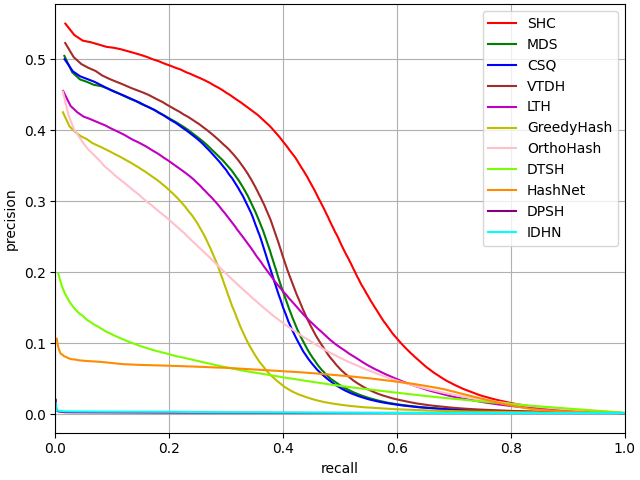}}
	\caption{The Precision-Recall curves of different methods on CIFAR100, Stanford Cars, and NABirds datasets (16, 32, and 64 bits).}
	\label{fig:precision-recall-curves}
\end{figure}

\begin{figure}[H]
	\centering\subfloat[16 bits, CIFAR-100]{\includegraphics[width=3.7cm]{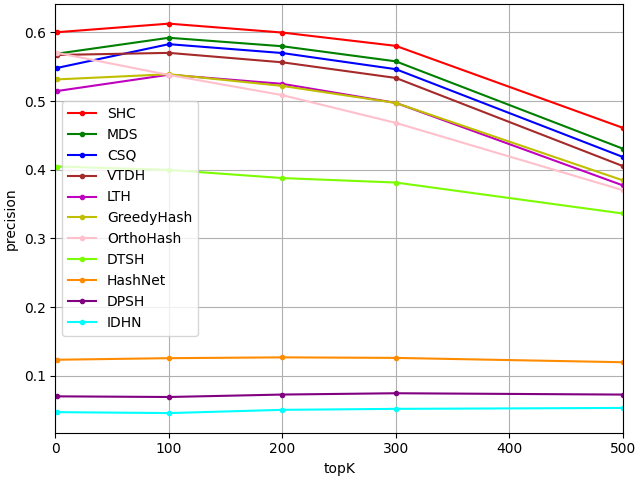}}\hspace{4mm}
	\centering\subfloat[32 bits, CIFAR-100]{\includegraphics[width=3.7cm]{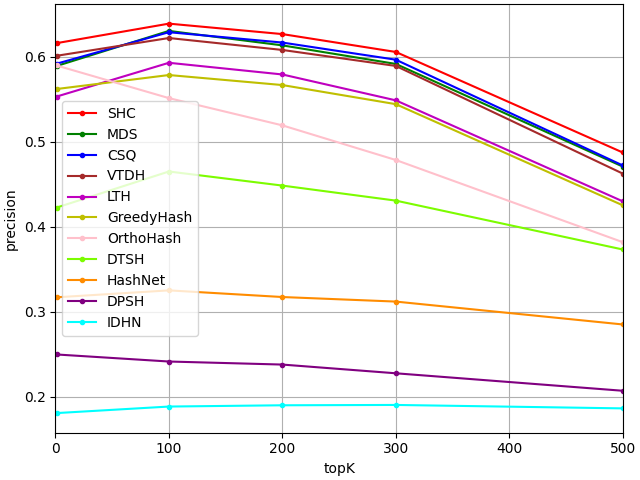}}\hspace{4mm}
	\centering\subfloat[64 bits, CIFAR-100]{\includegraphics[width=3.7cm]{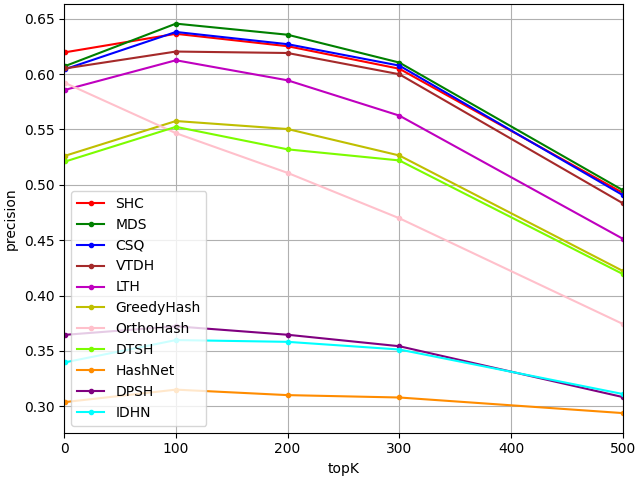}}
	
	\centering\subfloat[16 bits, Stanford Cars-A]{\includegraphics[width=3.7cm]{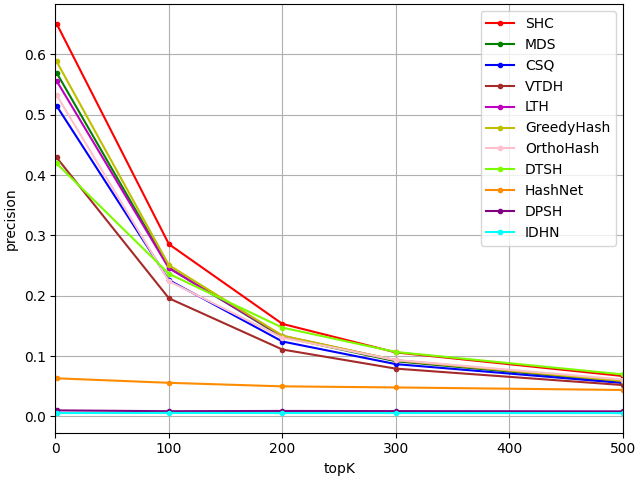}}\hspace{4mm}
	\centering\subfloat[32 bits, Stanford Cars-A]{\includegraphics[width=3.7cm]{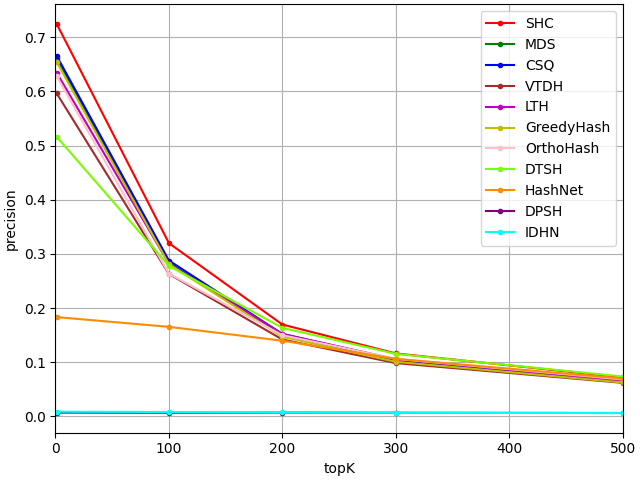}}\hspace{4mm}
	\centering\subfloat[64 bits, Stanford Cars-A]{\includegraphics[width=3.7cm]{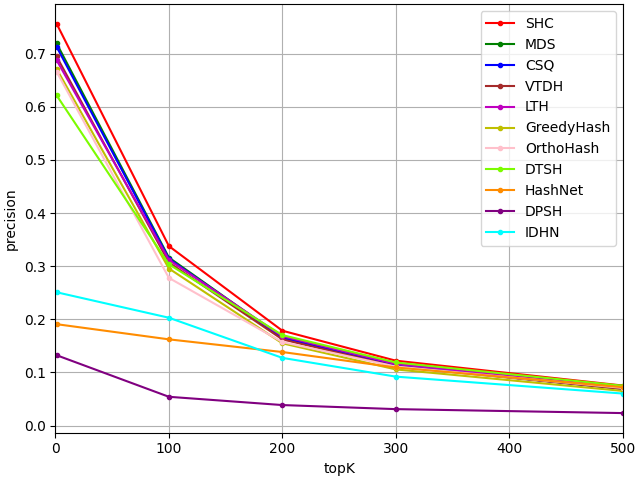}}
	
	\centering\subfloat[16 bits, Stanford Cars-B]{\includegraphics[width=3.7cm]{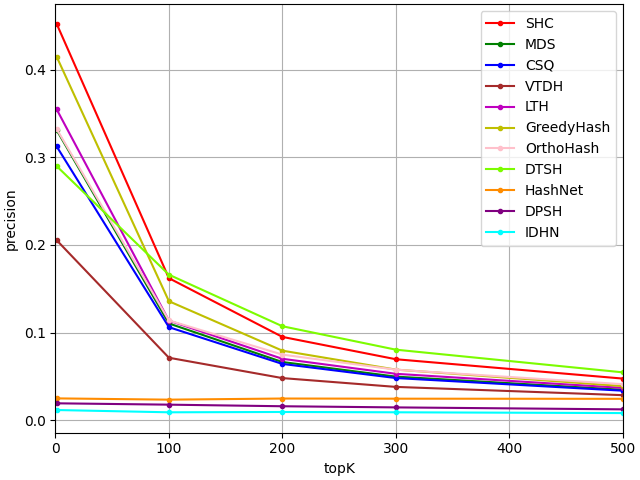}}\hspace{4mm}
	\centering\subfloat[32 bits, Stanford Cars-B]{\includegraphics[width=3.7cm]{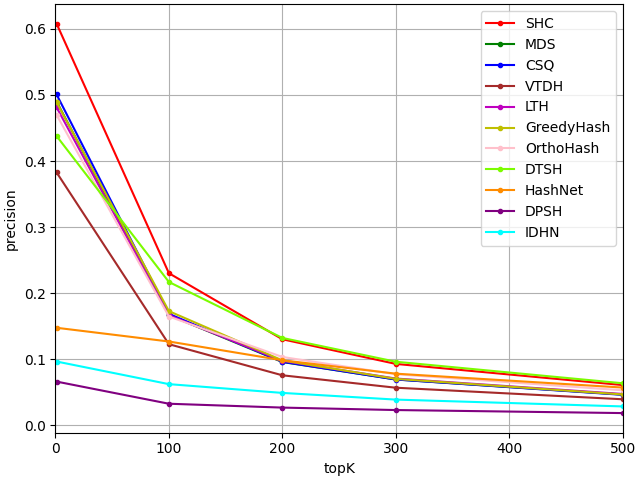}}\hspace{4mm}
	\centering\subfloat[64 bits, Stanford Cars-B]{\includegraphics[width=3.7cm]{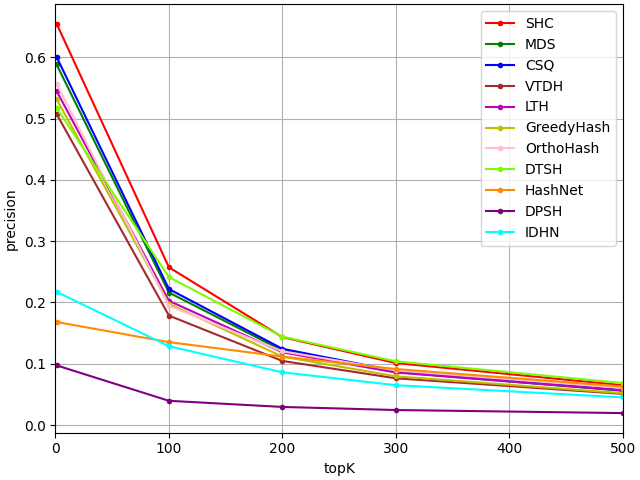}}
	
	\centering\subfloat[16 bits, NABirds-A]{\includegraphics[width=3.7cm]{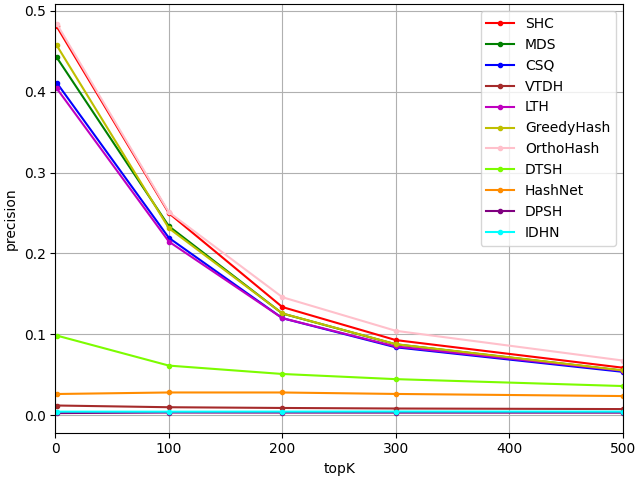}}\hspace{4mm}
	\centering\subfloat[32 bits, NABirds-A]{\includegraphics[width=3.7cm]{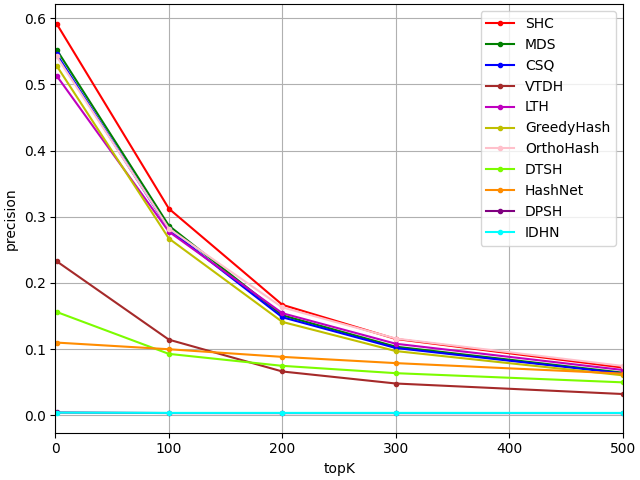}}\hspace{4mm}
	\centering\subfloat[64 bits, NABirds-A]{\includegraphics[width=3.7cm]{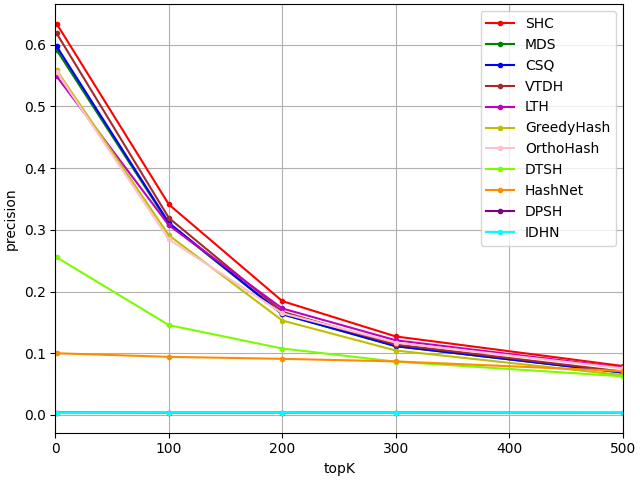}}
	
	\centering\subfloat[16 bits, NABirds-B]{\includegraphics[width=3.7cm]{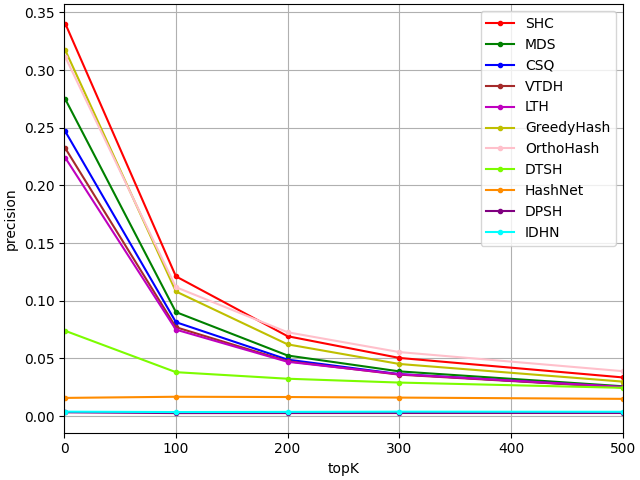}}\hspace{4mm}
	\centering\subfloat[32 bits, NABirds-B]{\includegraphics[width=3.7cm]{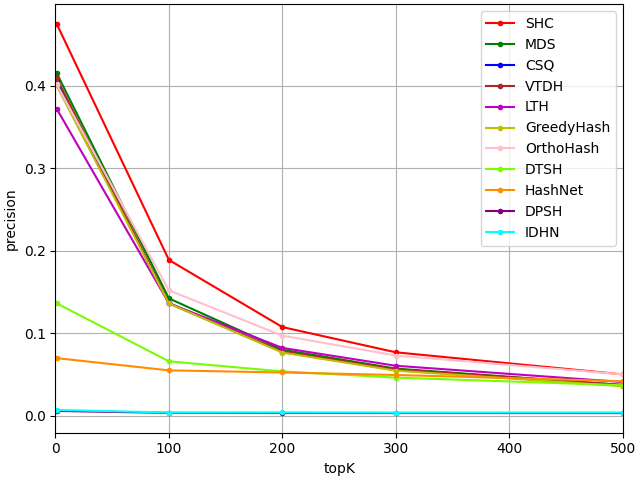}}\hspace{4mm}
	\centering\subfloat[64 bits, NABirds-B]{\includegraphics[width=3.7cm]{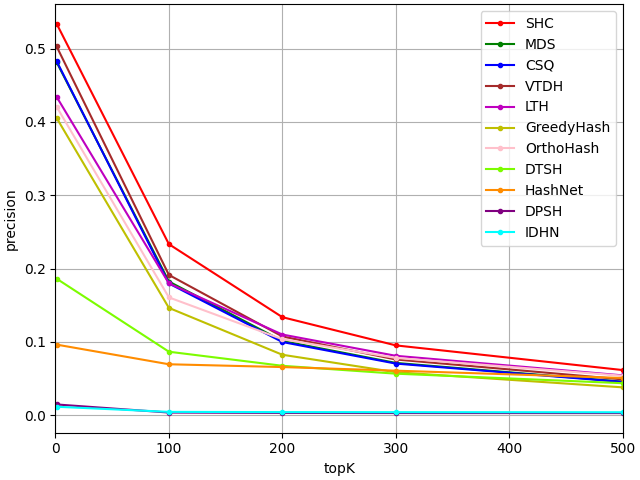}}
	\caption{The Precision@topK of different methods on CIFAR100, Stanford Cars, and NABirds datasets (16, 32, and 64 bits).}
	\label{fig:precision-at-topK}
\end{figure}

\begin{figure}[H]
	\centering\subfloat[16 bits, CIFAR-100]{\includegraphics[width=3.7cm]{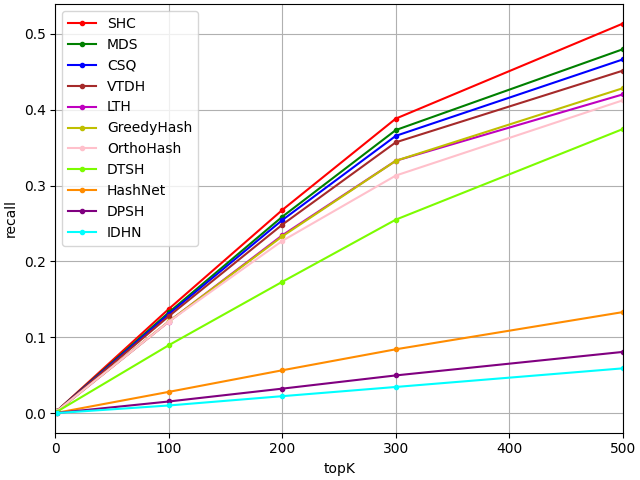}}\hspace{4mm}
	\centering\subfloat[32 bits, CIFAR-100]{\includegraphics[width=3.7cm]{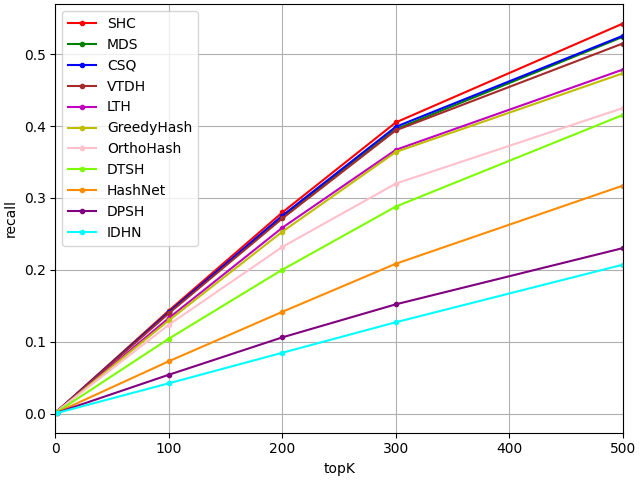}}\hspace{4mm}
	\centering\subfloat[64 bits, CIFAR-100]{\includegraphics[width=3.7cm]{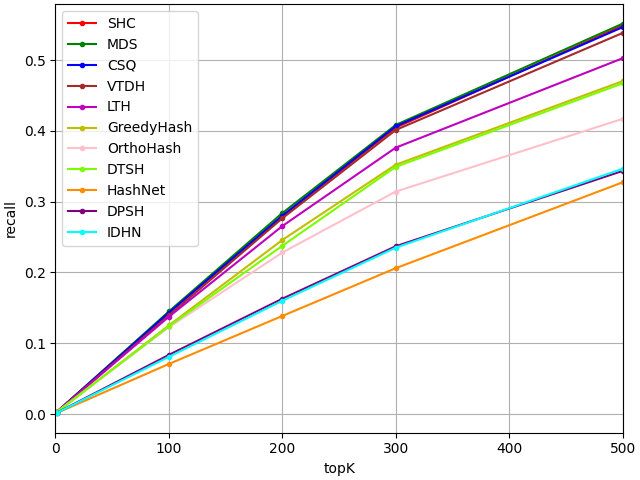}}
	
	\centering\subfloat[16 bits, Stanford Cars-A]{\includegraphics[width=3.7cm]{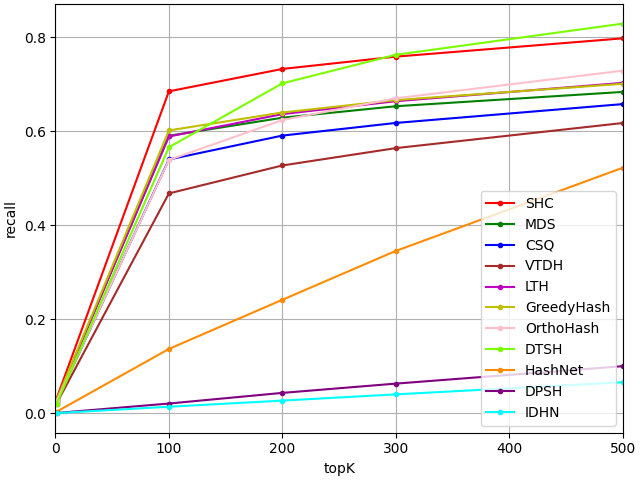}}\hspace{4mm}
	\centering\subfloat[32 bits, Stanford Cars-A]{\includegraphics[width=3.7cm]{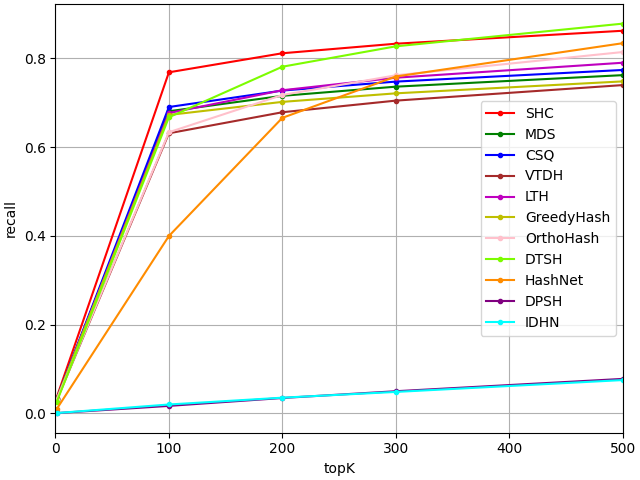}}\hspace{4mm}
	\centering\subfloat[64 bits, Stanford Cars-A]{\includegraphics[width=3.7cm]{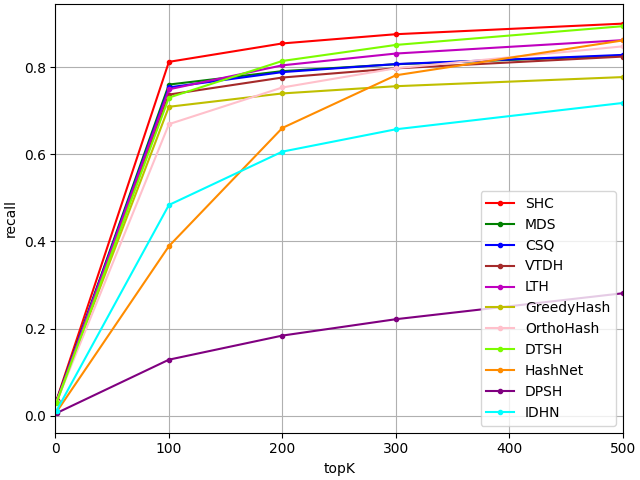}}
	
	\centering\subfloat[16 bits, Stanford Cars-B]{\includegraphics[width=3.7cm]{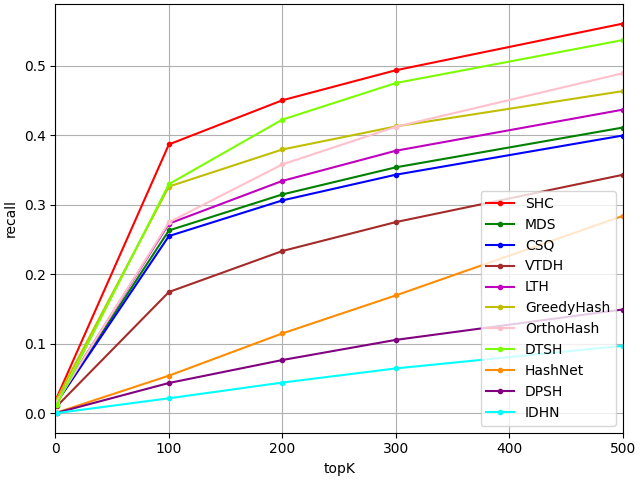}}\hspace{4mm}
	\centering\subfloat[32 bits, Stanford Cars-B]{\includegraphics[width=3.7cm]{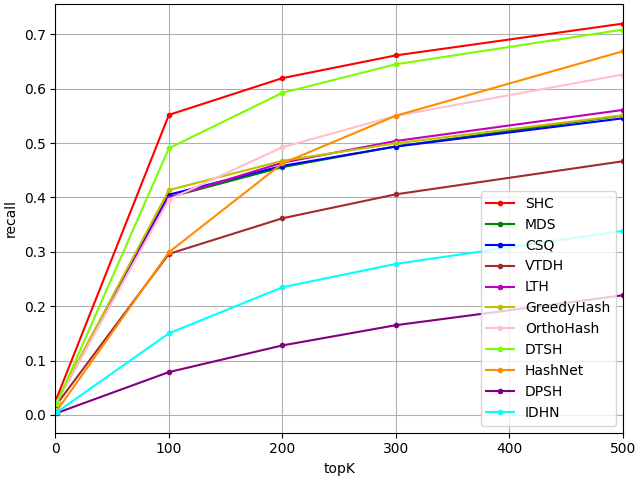}}\hspace{4mm}
	\centering\subfloat[64 bits, Stanford Cars-B]{\includegraphics[width=3.7cm]{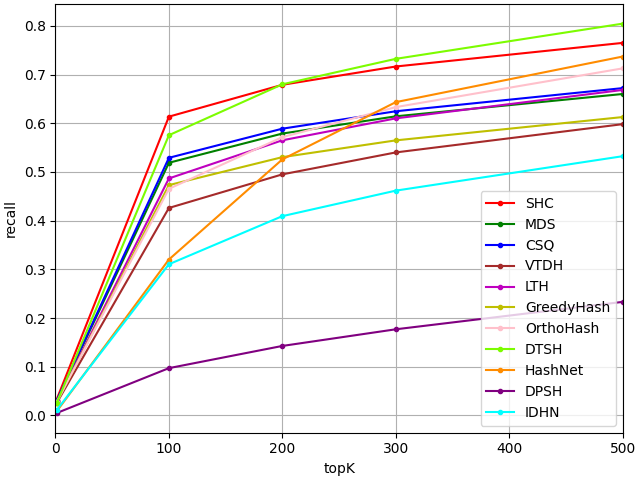}}
	
	\centering\subfloat[16 bits, NABirds-A]{\includegraphics[width=3.7cm]{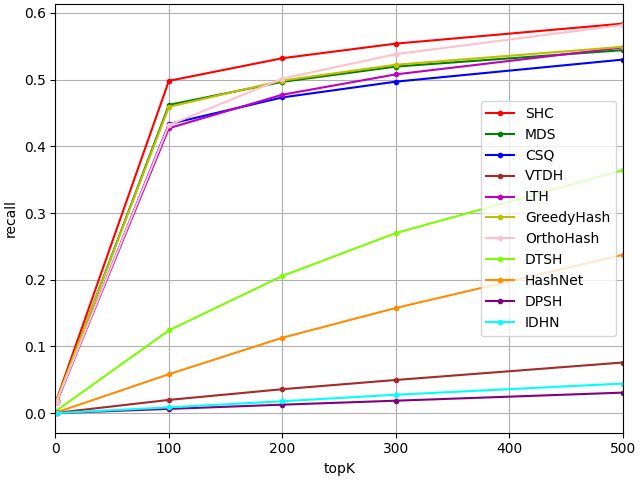}}\hspace{4mm}
	\centering\subfloat[32 bits, NABirds-A]{\includegraphics[width=3.7cm]{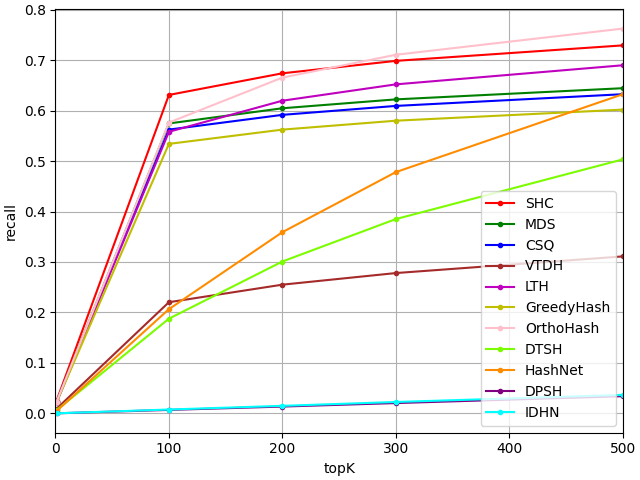}}\hspace{4mm}
	\centering\subfloat[64 bits, NABirds-A]{\includegraphics[width=3.7cm]{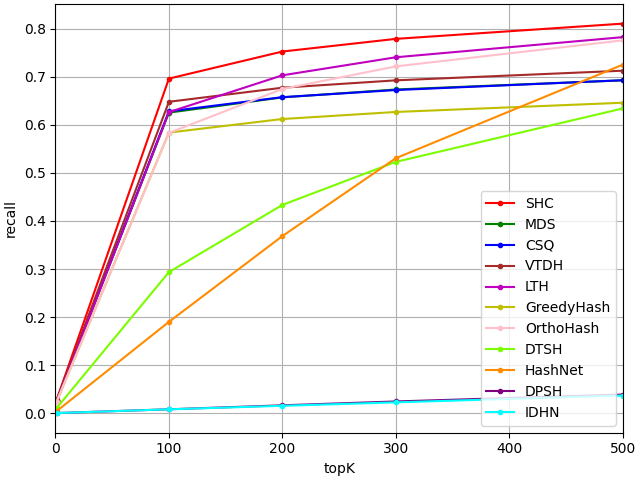}}
	
	\centering\subfloat[16 bits, NABirds-B]{\includegraphics[width=3.7cm]{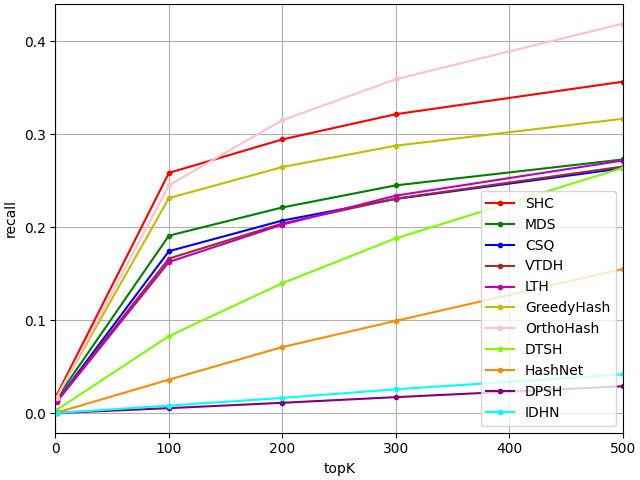}}\hspace{4mm}
	\centering\subfloat[32 bits, NABirds-B]{\includegraphics[width=3.7cm]{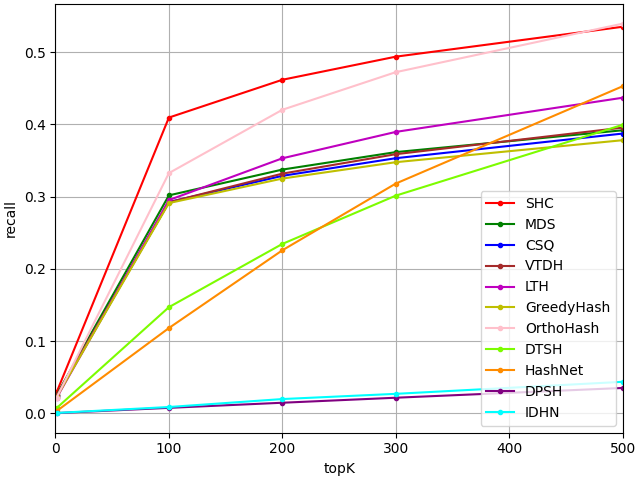}}\hspace{4mm}
	\centering\subfloat[64 bits, NABirds-B]{\includegraphics[width=3.7cm]{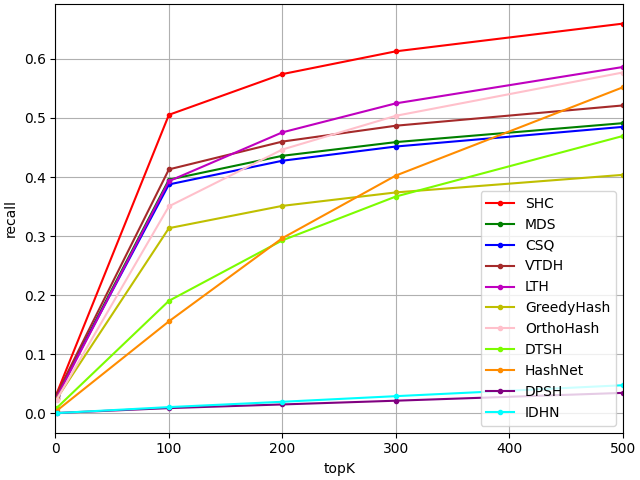}}
	\caption{The Recall@topK of different methods on CIFAR100, Stanford Cars, and NABirds datasets (16, 32, and 64 bits).}
	\label{fig:recall-at-topK}
\end{figure}

\subsection{Ablation Studies}

We conduct ablation experiments by removing each component individually, i.e., the semantic and distance constraints, and also explore how the widely-used BERT-based or our data-dependent similarity matrices affect the final retrieval performances of the proposed SHC method. 

\subsubsection{Effectiveness of the Minimal Distance Constraint}

To demonstrate the effectiveness of the distance constraint in optimization (\ref{optimization:objective}), we remove $\text{L}_{distance}$ and then there comes:
\begin{equation} \label{Ablation Objective Function}
\begin{split}
&\mathop{\min}_{\mathbf{H}} \text{L} = ||\mathbf{S} - \frac{1}{q}\mathbf{H}^T\mathbf{H}||^{2}_{F} \\
&\text{s.t.}
\begin{array}{ll}
    \mathbf{H} \in \{-1, +1\}^{q \times C}. &
\end{array}
\end{split}
\end{equation}

Similar to the optimization method used in Section~\ref{stage2}, we first rewrite the optimization (\ref{Ablation Objective Function}) as follows:
\begin{equation}
\begin{split}
\mathop{\min}_{\mathbf{H}} \text{L} & = ||\mathbf{S} - \frac{1}{q}\mathbf{H}^T\mathbf{M}||^{2}_{F} + \mu ||\mathbf{H} - \mathbf{M}||^{2}_{F} \\
&\text{s.t.}
\left \{
\begin{array}{ll}
    \mathbf{H} \in \{-1, +1\}^{q \times C}, & \\
    \mathbf{M} \in \mathbb{R}^{q \times C}.
\end{array}
\right.
\end{split}
\end{equation}

We can tackle the above problem by iteratively fixing $\mathbf{H}$ and optimizing $\mathbf{M}$, and then fixing $\mathbf{M}$ and optimizing $\mathbf{H}$.

\textbf{Update $\mathbf{M}$}:
\begin{equation}
\mathbf{M} = 
\left(\frac{1}{q}\mathbf{S} +  \mu \mathbf{I}\right) \mathbf{H}\left(\frac{1}{q^2}\mathbf{H}^T\mathbf{H} + \mu \mathbf{I}\right);
\end{equation}

\textbf{Update $\mathbf{H}$}:
\begin{equation}
\mathbf{H} = 
\text{sign}\left[(\frac{1}{q}\mathbf{S} + \mu \mathbf{I})\mathbf{M}(\mathbf{M}^T\mathbf{M})^{-1}\right].
\end{equation}

Now we obtained the hash centers without the distance constraint, i.e., SHC without $\text{L}_{distance}$. 
The experimental results are collected in Table \ref{tab:Ablation Studies 1}.

By comparing the MAP values of SHC and ``SHC without $\text{L}_{distance}$'', it is evident that SHC, with both semantic and distance constraints, outputs larger than in most cases or at least comparable MAP values with ``SHC without $\text{L}_{distance}$'', only with semantic constraint, which validates the effectiveness of the distance constraint.

\subsubsection{Effectiveness of the Semantic Constraint}
To demonstrate the effectiveness of the semantic constraint in optimization (\ref{optimization:objective}), we remove $\text{L}_{semantic}$, which is equivalent to MDS~\cite{MDS}. 
The experimental results are collected in Table \ref{tab:Ablation Studies 1}.

By comparing the MAP values of SHC and MDS, it is evident that SHC, with both semantic and distance constraints, outputs larger MAP values than MDS, only with distance constraint, which validates the effectiveness of the semantic constraint.

\subsubsection{Effectiveness of Different Similarity Matrices}
We also compare the MAP results of the method using the BERT~\cite{BERT}-based similarity matrix and our SHC using the data-dependent similarity matrix. 

The process of generating the similarity matrix using BERT is as below.
For each class, its label is denoted as ${l}_c$ ($c=1,2,\cdots,C$), then we can get its semantic embedding via the pre-trained BERT model via
$\mathbf{e}_c = \text{BERT}({l}_c)$,
based on which, the similarity between any two classes $\mathbf{s}_{ij}$ can be computed by their cosine values, i.e.,
\begin{equation} 
\mathbf{s}_{ij}=\text{cos}(\mathbf{e}_i, \mathbf{e}_j)=\frac{\mathbf{e}_i^T\mathbf{e}_j}{||\mathbf{e}_i||\times||\mathbf{e}_j||} ;
\end{equation}
then the similarity matrix $\mathbf{S}=(\mathbf{s}_{ij})^{C \times C}$ based on BERT is prepared (i.e., each $\mathbf{s}_{ij}$ is normalized to (-1,+1)) for subsequent semantic hash centers generation, after which it follows the same process as SHC. These two methods are marked as ``BERT'' v.s. ``SHC''. Their experimental results are collected in Table~\ref{tab:Ablation Studies 2}.

Obviously, the BERT-based similarity matrix calculated with class labels is not that effective as our data-dependent similarity matrix using the inherent distributions of the images. 
This is because that class labels are relatively static compared with specific image datasets, and they can not adapt well to dynamic image datasets (e.g., for fixed class labels, we can curate various distributed image datasets); in addition, for each class, it can be labeled with various texts, this also can lead to different similarities between classes but with the same image datasets. 

In sharp contrast, our designed data-dependent similarity matrix is just based on all the image samples, free from class labels, and hence can well automatically adapt to the distribution of various datasets.

\begin{table*}[!ht]
    \centering
    \caption{The MAP@topK (topK=100, 1000, and ALL) values of our SHC approach w/o semantic $\text{L}_{semantic}$ or distance $\text{L}_{distance}$ components. Note that \textbf{d} with \checkmark means the distance constraint; likewise, \textbf{S} with \checkmark means the semantic constraint.}
    \label{tab:Ablation Studies 1}
	\setlength{\tabcolsep}{0.4mm}
	{
		\renewcommand\arraystretch{1.5}
        \small
        \begin{tabular}{c | cc | ccc | ccc | ccc | ccc | ccc}
        \toprule
            \multirow{2}{*}{\textbf{MAP@topK}} & \multirow{2}{*}{\textbf{d}} & \multirow{2}{*}{\textbf{S}} & \multicolumn{3}{c|}{\textbf{CIFAR-100}} & \multicolumn{3}{c|}{\textbf{Stanford Cars-A}} & \multicolumn{3}{c|}{\textbf{Stanford Cars-B}} & \multicolumn{3}{c|}{\textbf{NABirds-A}} & \multicolumn{3}{c}{\textbf{NABirds-B}}\\ 
            \cmidrule(lr{.75em}){4-6} \cmidrule(lr{.75em}){7-9} \cmidrule(lr{.75em}){10-12} \cmidrule(lr{.75em}){13-15} \cmidrule(lr{.75em}){16-18} 
            & & & 16 bits & 32 bits & 64 bits & 16 bits & 32 bits & 64 bits & 16 bits & 32 bits & 64 bits & 16 bits & 32 bits & 64 bits & 16 bits & 32 bits & 64 bits \\
            \cmidrule(lr{.75em}){1-1} \cmidrule(lr{.75em}){2-3} \cmidrule(lr{.75em}){4-6} \cmidrule(lr{.75em}){7-9} \cmidrule(lr{.75em}){10-12} \cmidrule(lr{.75em}){13-15} \cmidrule(lr{.75em}){16-18} 

            \multirow{3}{*}{\makecell[c]{MAP@100}}
            & \checkmark & ~ & 0.6128  & 0.6495  & 0.6579  & 0.5804  & 0.6787  & 0.7393  & 0.3434  & 0.4828  & 0.5907  & 0.4837  & 0.5809  & 0.6234  & 0.3187  & 0.4401  & 0.4903   \\
            & ~ & \checkmark & 0.6337  & 0.6580  & \textbf{0.6726}  & 0.6268  & 0.7280  & 0.7647  & 0.3999  & 0.5712  & 0.6025  & 0.4926  & 0.5921  & {0.6563}  & 0.3429  & 0.4410  & \textbf{0.5437}   \\
            & \checkmark & \checkmark & \textbf{0.6360}  & \textbf{0.6629}  & 0.6672  & \textbf{0.6674}  & \textbf{0.7448}  & \textbf{0.7776}  & \textbf{0.4440}  & \textbf{0.5859}  & \textbf{0.6288}  & \textbf{0.4986}  & \textbf{0.6117}  & \textbf{0.6586}  & \textbf{0.3552}  & \textbf{0.4753}  & 0.5289   \\
            \cmidrule(lr{.75em}){1-1} \cmidrule(lr{.75em}){2-3} \cmidrule(lr{.75em}){4-6} \cmidrule(lr{.75em}){7-9} \cmidrule(lr{.75em}){10-12} \cmidrule(lr{.75em}){13-15} \cmidrule(lr{.75em}){16-18} 

            \multirow{3}{*}{\makecell[c]{MAP@1000}}
            & \checkmark & ~ & 0.5692  & 0.6064  & 0.6087  & 0.5654  & 0.6698  & 0.7355  & 0.2678  & 0.4026  & 0.5194  & 0.4756  & 0.5761  & 0.6218  & 0.2466  & 0.3678  & 0.4290   \\
            & ~ & \checkmark & 0.5830  & 0.6117  & \textbf{0.6208}  & 0.6100  & 0.7199  & 0.7632  & 0.3129  & 0.4818  & 0.5274  & 0.4827  & 0.5869  & {0.6567}  & 0.2669  & 0.3655  & \textbf{0.4646}   \\
            & \checkmark & \checkmark & \textbf{0.5936}  & \textbf{0.6166}  & {0.6175}  & \textbf{0.6525}  & \textbf{0.7379}  & \textbf{0.7762}  & \textbf{0.3541}  & \textbf{0.5033}  & \textbf{0.5573}  & \textbf{0.4878}  & \textbf{0.6072}  & \textbf{0.6608}  & \textbf{0.2813}  & \textbf{0.3991}  & 0.4575   \\
            \cmidrule(lr{.75em}){1-1} \cmidrule(lr{.75em}){2-3} \cmidrule(lr{.75em}){4-6} \cmidrule(lr{.75em}){7-9} \cmidrule(lr{.75em}){10-12} \cmidrule(lr{.75em}){13-15} \cmidrule(lr{.75em}){16-18}

            \multirow{3}{*}{\makecell[c]{MAP@ALL}} 
            & \checkmark & ~ & 0.4308  & 0.4826  & 0.5158  & 0.5509  & 0.6566  & 0.7280  & 0.2012  & 0.3260  & 0.4518  & 0.4646  & 0.5670  & 0.6135  & 0.1568  & 0.2590  & 0.3215   \\
            & ~ & \checkmark & 0.4602  & 0.5055  & {0.5209}  & 0.5989  & 0.7127  & 0.7565  & 0.2501  & 0.4280  & 0.4792  & 0.4720  & 0.5789  & {0.6521}  & 0.1774  & 0.2577  & 0.3950   \\
            & \checkmark & \checkmark & \textbf{0.4838}  & \textbf{0.5186}  & \textbf{0.5262}  & \textbf{0.6433}  & \textbf{0.7300}  & \textbf{0.7720}  & \textbf{0.2976}  & \textbf{0.4547}  & \textbf{0.5139}  & \textbf{0.4755}  & \textbf{0.6001}  & \textbf{0.6566}  & \textbf{0.1926}  & \textbf{0.3185}  & \textbf{0.3988}   \\
         
            \bottomrule
        \end{tabular}
    }
\end{table*}

\subsection{Convergence Analysis}

Fig.~\ref{fig:loss-map} exhibits the changes of losses and MAP values w.r.t. the training epochs on the curated datasets. 
Obviously, we could observe that across all datasets, the losses on the whole decrease as the number of training epochs increases. 
Owing to the utilization of cosine annealing learning rate\footnote{\url{https://pytorch.org/docs/stable/generated/torch.optim.lr_scheduler.CosineAnnealingLR.html}} in our experiments, the initial learning rate is relatively large. 
As the training epochs progressively rise up, the learning rate gradually decreases; 
therefore, the losses exhibit fluctuations in the early epochs, but as the learning rate becomes much smaller, the losses become more stable. 
 
The trends of MAP values align with these losses, showing a gradual increase overall. 
During the training when the losses fluctuate, the MAPs also exhibit fluctuations. 
As the learning rate decreases, the MAP values tend to stabilize.

\begin{figure}[H]
	\centering\subfloat[16 bits, CIFAR-100]{\includegraphics[width=3.7cm]{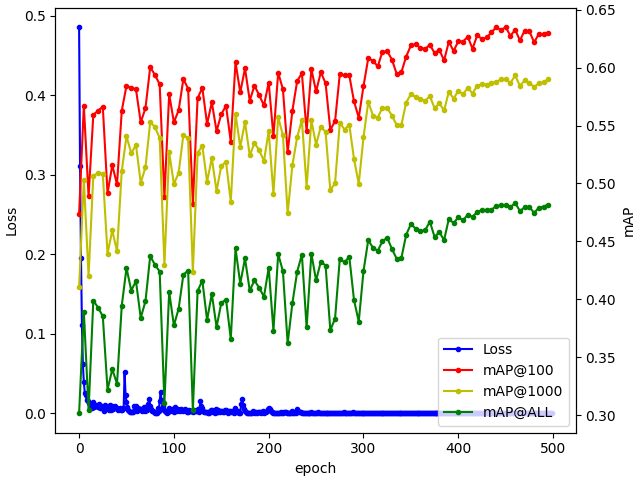}}\hspace{4mm}
	\centering\subfloat[32 bits, CIFAR-100]{\includegraphics[width=3.7cm]{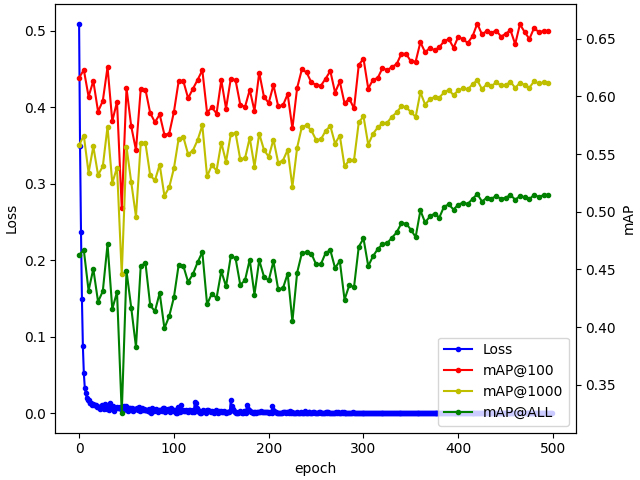}}\hspace{4mm}
	\centering\subfloat[64 bits, CIFAR-100]{\includegraphics[width=3.7cm]{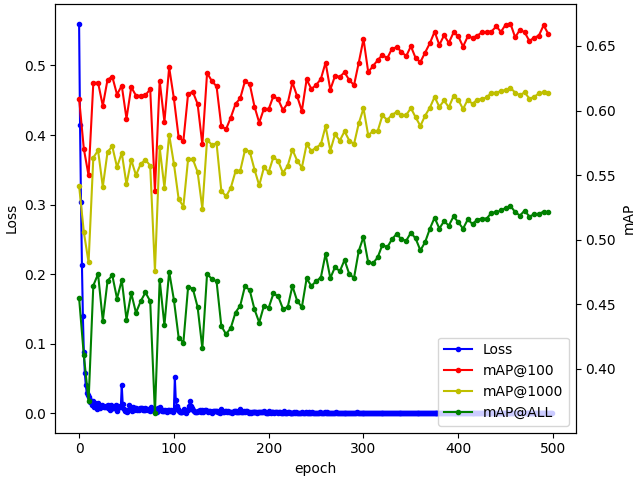}}
	
	\centering\subfloat[16 bits, Stanford Cars-A]{\includegraphics[width=3.7cm]{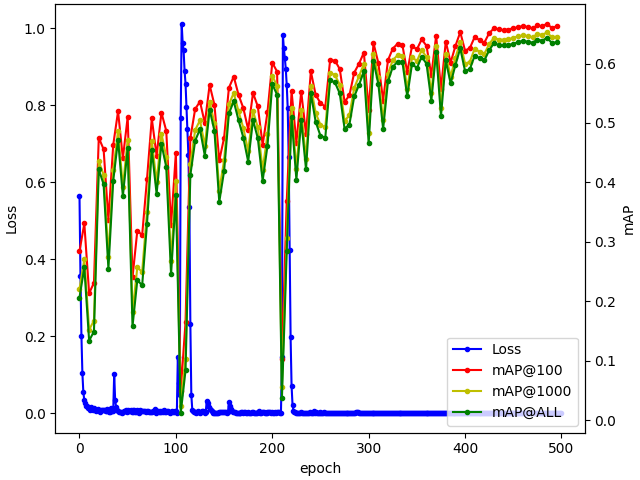}}\hspace{4mm}
	\centering\subfloat[32 bits, Stanford Cars-A]{\includegraphics[width=3.7cm]{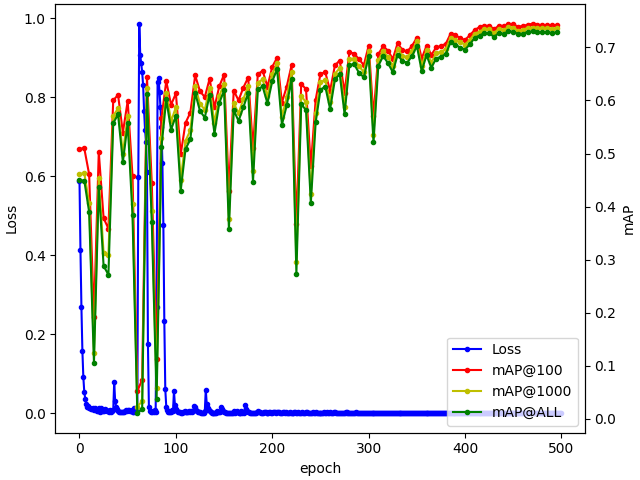}}\hspace{4mm}
	\centering\subfloat[64 bits, Stanford Cars-A]{\includegraphics[width=3.7cm]{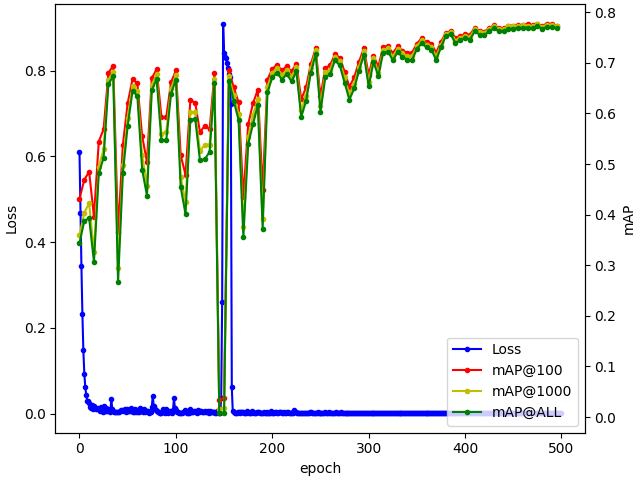}}
	
	\centering\subfloat[16 bits, Stanford Cars-B]{\includegraphics[width=3.7cm]{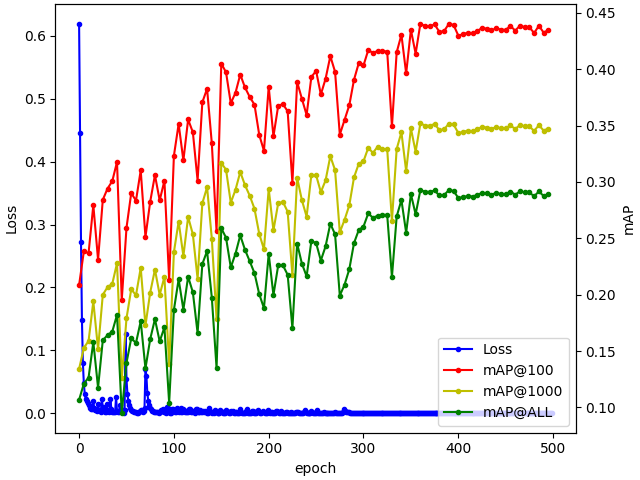}}\hspace{4mm}
	\centering\subfloat[32 bits, Stanford Cars-B]{\includegraphics[width=3.7cm]{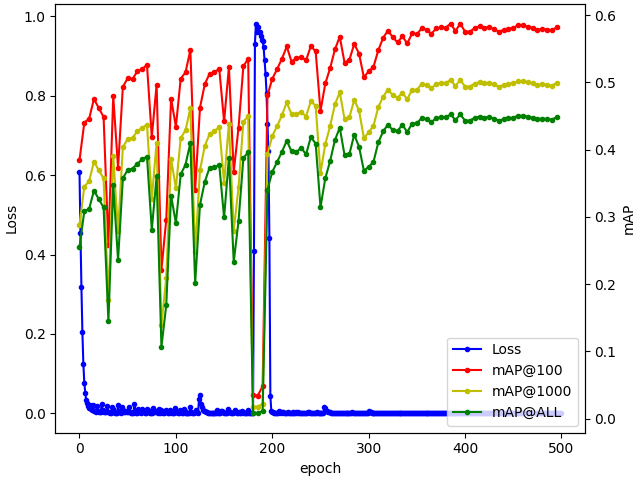}}\hspace{4mm}
	\centering\subfloat[64 bits, Stanford Cars-B]{\includegraphics[width=3.7cm]{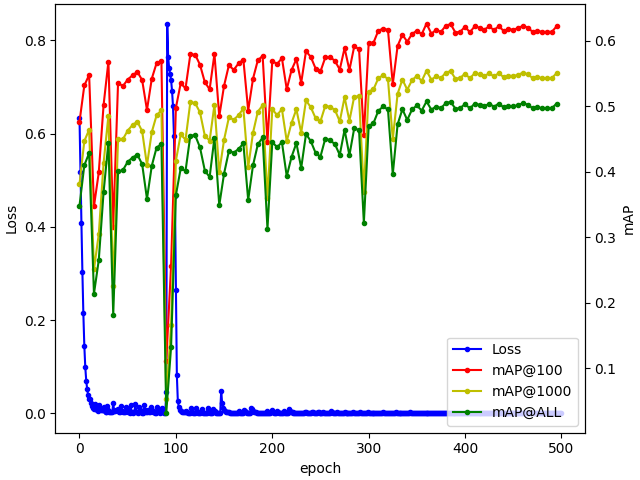}}
	
	\centering\subfloat[16 bits, NABirds-A]{\includegraphics[width=3.7cm]{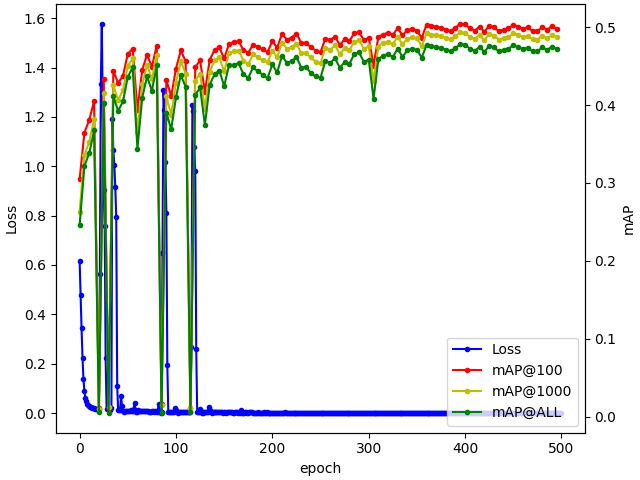}}\hspace{4mm}
	\centering\subfloat[32 bits, NABirds-A]{\includegraphics[width=3.7cm]{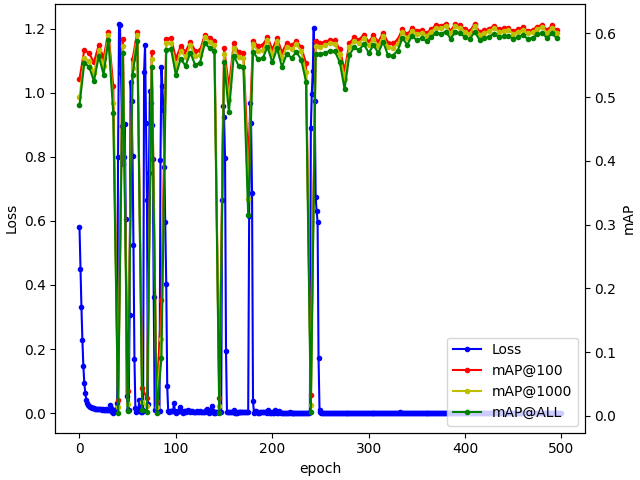}}\hspace{4mm}
	\centering\subfloat[64 bits, NABirds-A]{\includegraphics[width=3.7cm]{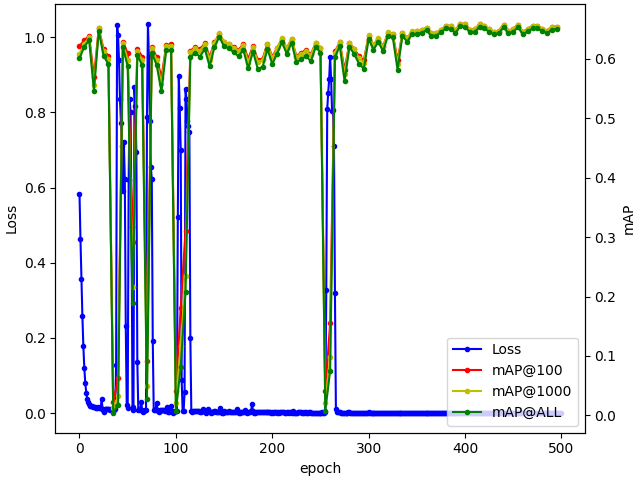}}
	
	\centering\subfloat[16 bits, NABirds-B]{\includegraphics[width=3.7cm]{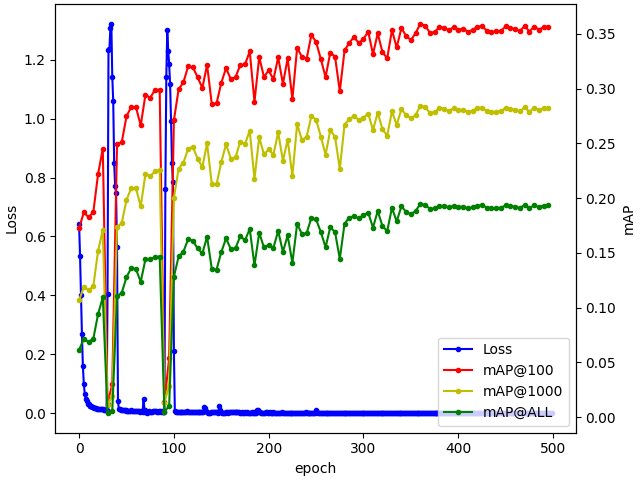}}\hspace{4mm}
	\centering\subfloat[32 bits, NABirds-B]{\includegraphics[width=3.7cm]{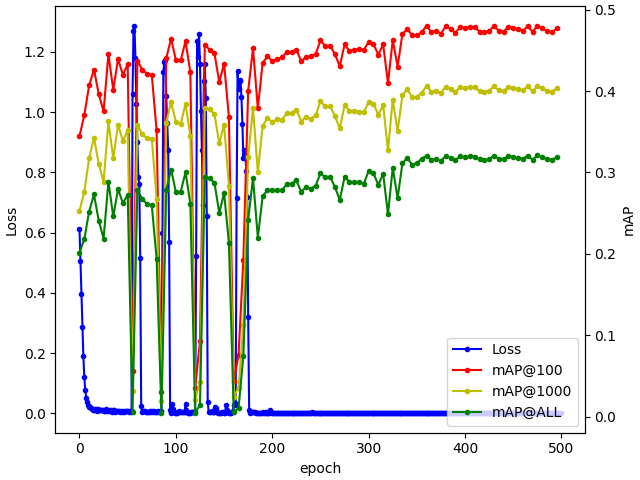}}\hspace{4mm}
	\centering\subfloat[64 bits, NABirds-B]{\includegraphics[width=3.7cm]{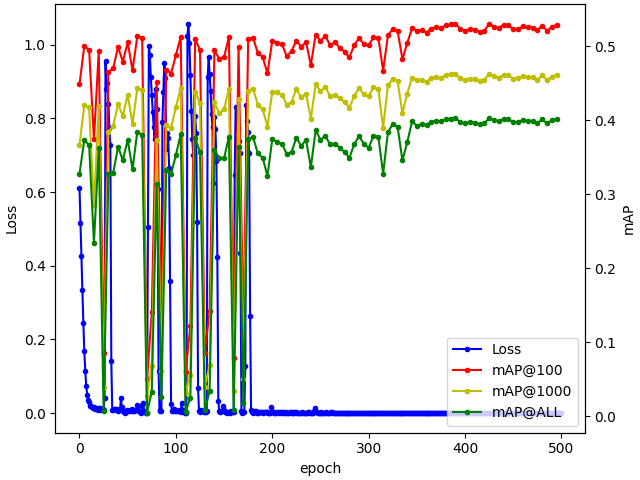}}
	\caption{The mAP@topK (topK=100, 1000, ALL) and loss curves w.r.t. training iterations on CIFAR100, Stanford Cars, and NABirds datasets.}
	\label{fig:loss-map}
\end{figure}

\begin{table}[ht]
    \centering
    \caption{Comparisons of the retrieval performance between the BERT-based and  data-dependent similarity matrices for semantic hash centers generation.}
    \label{tab:Ablation Studies 2}
	\setlength{\tabcolsep}{0.22mm}
	{
		\renewcommand\arraystretch{1.5}
        \small
        \begin{tabular}{c | c | ccc | ccc | ccc | ccc | ccc}
        \toprule
            \multirow{2}{*}{\textbf{MAP@topK}} & \multirow{2}{*}{\textbf{S type}} & \multicolumn{3}{c|}{\textbf{CIFAR-100}} & \multicolumn{3}{c|}{\textbf{Stanford Cars-A}} & \multicolumn{3}{c|}{\textbf{Stanford Cars-B}} & \multicolumn{3}{c|}{\textbf{NABirds-A}} & \multicolumn{3}{c}{\textbf{NABirds-B}}\\ 
            \cmidrule(lr{.75em}){3-5} \cmidrule(lr{.75em}){6-8} \cmidrule(lr{.75em}){9-11} \cmidrule(lr{.75em}){12-14} \cmidrule(lr{.75em}){15-17} 
            & & 16 bits & 32 bits & 64 bits & 16 bits & 32 bits & 64 bits & 16 bits & 32 bits & 64 bits & 16 bits & 32 bits & 64 bits & 16 bits & 32 bits & 64 bits \\
            \cmidrule(lr{.75em}){1-1} \cmidrule(lr{.75em}){2-2} \cmidrule(lr{.75em}){3-5} \cmidrule(lr{.75em}){6-8} \cmidrule(lr{.75em}){9-11} \cmidrule(lr{.75em}){12-14} \cmidrule(lr{.75em}){15-17} 

            \multirow{2}{*}{\makecell[c]{MAP@100}}
            & BERT & 0.6154  & 0.6497  & 0.6620  & 0.5911  & 0.6798  & 0.7348  & 0.3384  & 0.4830  & 0.5673  & 0.4509  & 0.5279  & 0.6222  & 0.3189  & 0.4427  & 0.4976   \\
            & SHC & \textbf{0.6360}  & \textbf{0.6629}  & \textbf{0.6672}  & \textbf{0.6674}  & \textbf{0.7448}  & \textbf{0.7776}  & \textbf{0.4440}  & \textbf{0.5859}  & \textbf{0.6280}  & \textbf{0.4986}  & \textbf{0.6117}  & \textbf{0.6586}  & \textbf{0.3552}  & \textbf{0.4753}  & \textbf{0.5289}   \\
            \cmidrule(lr{.75em}){1-1} \cmidrule(lr{.75em}){2-2} \cmidrule(lr{.75em}){3-5} \cmidrule(lr{.75em}){6-8} \cmidrule(lr{.75em}){9-11} \cmidrule(lr{.75em}){12-14} \cmidrule(lr{.75em}){15-17}

            \multirow{2}{*}{\makecell[c]{MAP@1000}} 
            & BERT & 0.5664  & 0.6025  & 0.6151  & 0.5757  & 0.6712  & 0.7297  & 0.2557  & 0.3950  & 0.4851  & 0.4445  & 0.5699  & 0.6225  & 0.2498  & 0.3720  & 0.4307   \\
            & SHC & \textbf{0.5936}  & \textbf{0.6166}  & \textbf{0.6175}  & \textbf{0.6525}  & \textbf{0.7379}  & \textbf{0.7762}  & \textbf{0.3541}  & \textbf{0.5033}  & \textbf{0.5573}  & \textbf{0.4878}  & \textbf{0.6072}  & \textbf{0.6608}  & \textbf{0.2813}  & \textbf{0.3991}  & \textbf{0.4575}   \\
            \cmidrule(lr{.75em}){1-1} \cmidrule(lr{.75em}){2-2} \cmidrule(lr{.75em}){3-5} \cmidrule(lr{.75em}){6-8} \cmidrule(lr{.75em}){9-11} \cmidrule(lr{.75em}){12-14} \cmidrule(lr{.75em}){15-17}

            \multirow{2}{*}{\makecell[c]{MAP@ALL}} 
            & BERT & 0.4263  & 0.4772  & 0.5096  & 0.5625  & 0.6593  & 0.7212  & 0.1965  & 0.3227  & 0.4196  & 0.4315  & 0.5509  & 0.6151  & 0.1613  & 0.2659  & 0.3370   \\
            & SHC & \textbf{0.4838}  & \textbf{0.5186}  & \textbf{0.5262}  & \textbf{0.6433}  & \textbf{0.7300}  & \textbf{0.7720}  & \textbf{0.2976}  & \textbf{0.4547}  & \textbf{0.5139}  & \textbf{0.4755}  & \textbf{0.6001}  & \textbf{0.6566}  & \textbf{0.1926}  & \textbf{0.3185}  & \textbf{0.3988}   \\
    
            \bottomrule
        \end{tabular}
    }
\end{table}

\section{Conclusion}

In this paper, we expand upon the newly emerged concept of ``hash centers'' to develop the more explainable ``\emph{semantic} hash centers'', proposing a novel deep hashing approach named SHC, structured around a three-stage learning procedure.
First, we develop a data-dependent, rather than label-based, pairwise similarity calculation method that can well adapt to various data distributions.
Second, we not only impose semantic constraints on the hash centers, but also set them apart as far as possible, and meanwhile compute a lower bound for the minimal distance between any two hash centers. 
Through the strategic use of proxy variables and by utilizing the ALM optimization framework, we are able to gradually decompose the complex optimization problem into several manageable sub-problems and solve them efficiently.
Third, we employ a classic deep hashing network that is supervised using the semantic hash centers derived above.
Extensive experiments on multiple datasets demonstrate that our SHC method outperforms several current state-of-the-art deep hashing techniques in image retrieval tasks.
Furthermore, ablation studies confirm the effectiveness of our devised data-dependent similarity matrix and the utilized optimization algorithm for generating semantic hash centers. 

\section*{Acknowledgments}

This work is supported in part by 
National Natural Science Foundation of China (Grant No.~62372054, 62006005) and 
National Key Research and Development Program of China (Grant No.~2022YFC3302200).

\bibliographystyle{ACM-Reference-Format}
\bibliography{sample-base}

\appendix

\end{document}